\UseRawInputEncoding
\documentclass[review]{elsarticle}

\usepackage{lineno,hyperref}
\usepackage{booktabs} 
\usepackage{float}
\usepackage{amsfonts}
\usepackage{amsmath}
\usepackage{graphicx}
\usepackage{subfig}
\usepackage{algorithm}
\usepackage{algorithmic}

\journal{Neurocomputing}

\usepackage{numcompress}

\begin{document}

\begin{frontmatter}

\title{To Boost or not to Boost: On the Limits of Boosted Neural Networks}



\author[address1,address2]{Sai Saketh Rambhatla}

\author[address2]{Michael Jones\corref{mycorrespondingauthor}}
\cortext[mycorrespondingauthor]{Corresponding author}

\author[address1,address3]{Rama Chellappa}

\address[address1]{Dept. of Electrical and Computer Engineering, University of Maryland, College Park, Maryland, USA,\{rssaketh,rama\}@umiacs.umd.edu}

\address[address2]{Mitsubishi Electric Research Laboratories (MERL), 201 Broadway, Cambridge, MA 02139, USA, mjones@merl.com}
\address[address3]{Dept. of Electrical and Computer Engineering, Johns Hopkins University, Maryland, USA, rchella4@jhu.edu}
\begin{abstract}
 Boosting is a method for finding a highly accurate hypothesis by linearly combining many ``weak" hypotheses, each of which may be only moderately accurate.  Thus, boosting is a method for learning an ensemble of classifiers.  While boosting has been shown to be very effective for decision trees, its impact on neural networks has not been extensively studied.  We prove one important difference between sums of decision trees compared to sums of convolutional neural networks (CNNs) which is that a sum of decision trees cannot be represented by a single decision tree with the same number of parameters while a sum of CNNs can be represented by a single CNN.  Next, using standard object recognition datasets, we verify experimentally the well-known result that a boosted ensemble of decision trees usually generalizes much better on testing data than a single decision tree with the same number of parameters.  In contrast, using the same datasets and boosting algorithms, our experiments show the opposite to be true when using neural networks (both CNNs and multilayer perceptrons (MLPs)).  We find that a single neural network usually generalizes better than a boosted ensemble of smaller neural networks with the same total number of parameters.
\end{abstract}


\end{frontmatter}


\section{Introduction}

We are motivated to study the application of boosting \cite{Schapire1990,Freund1995} to neural networks (CNNs) because of the great success boosting has had in conjunction with decision trees and other classifiers.  Before the recent explosion of research on neural networks and especially convolutional neural networks (CNNs), boosted decision trees were considered to be state of the art.  In fact, the well-known statistician, Leo Breiman, once called boosted decision trees "the best off-the-shelf classifier in the world" \cite{FriedmanEtAl2000}.  

The basic idea of boosting is to form a "strong" classifier (one with high accuracy) using a linear combination (an ensemble) of "weak" classifiers.  The only requirement of a weak classifier is that it has accuracy slightly better than chance.  AdaBoost~\cite{FreundSchapire1997} is probably the best known boosting algorithm and has been widely used by machine learning researchers.  AdaBoost maintains a set of weights per training example and requires that any weak learning algorithm returns a classifier that has a weighted accuracy better than chance on the training examples.  On each round of boosting, the weight on each example is updated with a specific equation that gives less weight to examples the weak classifier got right and more weight to examples it got wrong.  The next weak classifier will be forced to classify more of the incorrect examples correctly.  AdaBoost has been proven to reduce the training error as more weak classifiers are added to the ensemble~\cite{FreundSchapire1997}.

Since boosting can be applied to any classifier, it makes sense to try to combine the power of neural networks with the power of boosting.  The idea of boosting many small CNNs to create a very accurate classifier while avoiding the trial-and-error search for better network architectures is an exciting possibility.  Other researchers have had the same motivation \cite{SchwenkBengio2000,MederaBabinec2009,Deep_Incremental_Boosting,MoghimiEtAl2016}.  In the past, researchers have shown that a boosted ensemble of decision trees or neural networks improves accuracy as the number of decision trees or neural networks increases.  An examination of past Kaggle challenge winners shows the success of neural network ensembles.  Unlike past work, however, we look at the accuracy of an ensemble of classifiers compared to the accuracy of a single classifier of the same type with the {\it same number of total parameters}.  When comparing classifiers of the same type with different numbers of parameters and trained with different algorithms, it is not possible to know whether any accuracy advantages are due to the difference in number of parameters or the difference in training algorithm.  Furthermore, classifiers with more parameters are computationally more expensive so it makes sense to study how to maximize accuracy within a certain parameter budget.  In the case of decision trees, we show that a boosted ensemble is usually much more accurate than a single decision tree with the same number of total parameters.  Surprisingly, however, the same may not hold for neural networks.  The main contribution of this paper is to present theoretical arguments and empirical evidence that a single large neural network is usually more accurate than a boosted ensemble of neural networks with the same number of total parameters.  In terms of accuracy and assuming sufficient training data, it appears that one is better off training a single large network than an ensemble of small networks. 

The remainder of the paper is organized as follows.  Past work on boosting both using decision trees and neural networks is discussed in Section \ref{sec:background}.  An overview of the basic boosting method is given in Section \ref{sec:Adaboost} as well as details on the specific version of multi-class boosting that we use.  In Section \ref{sec:boosted_DTs}, we discuss properties of boosted decision trees and prove that an ensemble of decision trees cannot be represented by a single, larger decision tree.  In Section \ref{sec:boosted_CNNs}, we discuss the properties of boosted CNNs and prove that a sum of CNNs can be represented by a single CNN with the same total number of parameters in contrast to the case with boosted decision trees.  In Section \ref{sec:DTexp}, we present experiments on boosted decision trees using three well-known object recognition datasets.  Analogous experiments are presented in Section 
\ref{sec:CNN_experiments} for boosted neural networks.  These experiments confirm that while boosting makes sense for decision trees, it may not be a win for neural networks.  Finally, in Section \ref{sec:conclusions} we give some conclusions and speculate on why boosting is not as effective for neural networks.

\section{Background and Related Work}
\label{sec:background}

Boosting was first developed in the early to mid 1990's by Schapire and Freund \cite{Schapire1990,Freund1995,FreundSchapire1997,FreundSchapire1996} as a method of creating more accurate classifiers from linear combinations (ensembles) of simpler classifiers.  The AdaBoost algorithm~\cite{FreundSchapire1996} was the first practical boosting algorithm to come from this initial work.  This early work opened up a rich line of further research that explored the potential of boosting and tried to better understand its success.  Some of the key follow on papers are by Friedman, Hastie and Tibshirani~\cite{Friedman1999,FriedmanEtAl2000} who tried to explain AdaBoost in terms of well-understood techniques in statistics.  Friedman proposed gradient boosting~\cite{Friedman1999} as a method of learning a regression function in a stagewise fashion by fitting a function that approximates the residual errors of the training data at each stage and then using a weighted sum of the functions learned at each stage.  Mason et al.~\cite{MasonEtAl1999} further showed the connections between boosting and gradient descent function optimization.  Friedman et al.~\cite{FriedmanEtAl2000} derived the LogitBoost algorithm as a way of understanding AdaBoost in terms of additive logistic regression.
Another notable paper by Zhu et al.~\cite{ZhuEtAl2006} proposed a multi-class boosting algorithm called SAMME (Stagewise Additive Modeling
using a Multi-class Exponential loss function) that is similar to AdaBoost but differs in how it computes the weight of each weak classifier and how it updates the example weights in each round of boosting.

In the early work on boosting, decision trees were typically used as weak classifiers despite the fact that boosting applies to any type of classifier.  Some examples of early work on boosting decision trees include \cite{FreundSchapire1996,Quinlan1996,SchapireSinger1999,ZhuEtAl2006}.  These and many more works demonstrated the interesting result that larger and larger ensembles with growing numbers of parameters usually did not overfit the training data.  In other words, testing error continued to decrease along with training error as the number of weak classifiers increased.   In contrast, it was well known that as the size (number of parameters) of a single decision tree increased, it would eventually overfit, i.e. the testing error would eventually increase as the training error decreased to zero \cite{BreimanEtAl1984}.  Combating this problem is a major reason why boosted decision trees are so useful - they allow the training of classifiers with large numbers of parameters, usually without overfitting.

There have also been many papers on applying boosting to neural networks.  Early papers applied boosting to multilayer perceptrons (MLPs) such as \cite{DruckerEtAl1993,SchwenkBengio1997,SchwenkBengio2000,BanfieldEtAl2007} and showed that boosted ensembles of MLPs got significantly lower errors rates than a single MLP on problems such as character recognition and the UCI classification datasets.  More recently, boosting has been applied to deep convolutional neural networks \cite{MederaBabinec2009,LeeEtAl2015,MoghimiEtAl2016,MoscaMagoulas2017,Deep_Incremental_Boosting}.  Many of these papers have proposed different boosting algorithms that are tailored for CNNs, but they all result in a linear combination of CNNs. For example, Moghimi et al.~\cite{MoghimiEtAl2016} derived a boosting algorithm for CNNs for multi-class classification problems in which the CNN weak learner was trained to estimate the example weights since this was shown to be equivalent to taking a step in the direction of the negative gradient of the error.  Mosca and Magoulas~\cite{MoscaMagoulas2017} use the multiclass AdaBoost.M2 algorithm of Freund and Schapire~\cite{FreundSchapire1996} but initialize each subsequent CNN weak classifier using the weights learned in the previous boosting round which leads to both higher accuracy and faster network training times.

Recent advances in boosting are mainly focused on very fast and efficient implementations of gradient boosting, for example XGBoost~\cite{ChenGuestrin2016} and CatBoost~\cite{DorogushEtAl2017} are popular gradient boosting packages.  For a more comprehensive survey of various boosting algorithms, please see \cite{HeEtAl2019,Schapire2003}.

Past work on boosting neural networks has not analyzed whether an ensemble of MLPs or CNNs is a "win" in terms of decreasing the testing error below what is achievable with a single network with the same number of total parameters as in an ensemble.  Unlike decision trees which are very susceptible to overfitting, this is less of a problem with highly over-parameterized CNNs \cite{AllenZhuEtAl2019,BelkinEtAl2019}.  Typically, training CNNs with more parameters does not result in worse generalization error.  Although it is certainly true that CNNs can overfit especially when trained on very small datasets (such as hundreds of examples), one of the reasons they have revolutionized machine learning is because they usually generalize well even when only trained on thousands of examples with networks containing millions of parameters.  While this phenomenon may be counter-intuitive, there is a growing body of research attempting to explain this surprising behavior \cite{ZhangEtAl2017,PoggioEtAl2018,BelkinEtAl2019,AllenZhuEtAl2019} but the fact remains that networks with more parameters tend to generalize better (see Figures \ref{fig:ensemble}, \ref{fig:baby}, \ref{fig:mlp}, \ref{fig:vgg8}, for example).

So the incentive to use boosting to avoid overfitting may not be so important for neural networks.  Then the question becomes whether boosted ensembles of neural networks can achieve better accuracy than simply using standard deep learning methods on a single, large neural network.  This is the main question we will address here.

\section{Boosting/AdaBoost}
\label{sec:Adaboost}

Boosting is a meta-learning algorithm that uses a base learning algorithm (also called a weak learner) to build an ensemble of (weak) classifiers such that the ensemble classifier (also called a strong classifier) is more accurate than any of the weak classifiers.  The main idea is to assign each training example a weight and use the weak learner to return a weak classifier that minimizes the {\em weighted} error.  The example weights are initially uniform but are updated on each round of boosting so as to give more weight to examples that the previous weak classifier got wrong and less weight to examples that it gets right.  This forces the next weak classifier to get more of the examples right that the previous weak classifier got wrong.  The final strong classifier gives more weight to the more accurate weak classifiers.  Many versions of this basic boosting method have been published with the main differences being how the margin for each example is computed and the weight assigned to each weak classifier in the ensemble.

In our experiments, we tested different multi-class versions of boosting, which all worked similarly, but we found the following version (Algorithm \ref{alg:boost}) which is slightly modified from the AdaBoost.M1 algorithm of Freund and Schapire~\cite{FreundSchapire1997} to perform best.  We follow the notation used in the generalized version of AdaBoost described in
Schapire and Singer (\cite{SchapireSinger1999}). 

In our multi-class version, weak classifiers output a probability distribution over output classes represented as a $C$ element real-valued vector.  On each round, $t$,  of boosting, a weak classifier is found that minimizes a weighted error over examples based on the current weights of examples, $D_t$.  In our case, the weak learning algorithm is either the decision tree training algorithm or a neural network training algorithm.  For training neural networks, we sample batches according to the distribution $D_t$ so that the neural network training algorithm does not have to use example weights directly.

Once a weak classifier is trained, we compute the margin of each example which is essentially how correct the weak classifier is on each example.  We use a margin, $\mathcal{M}$, similar to the one defined in \cite{SaberianVasconcelos2011}, which is equal to the probability of the correct class minus the average probability of the other classes on the training examples.
\begin{equation}
\label{eq:margin}
m_i = \mathcal{M}(\mathbf{h_t}(x_i), y_i) = \mathbf{p_i}[y_i] - \frac{1}{C-1}(1-\mathbf{p_i}[y_i])
\end{equation}
where $y_i$ is the correct label for input sample $x_i$, $C$ is the number of classes and $\mathbf{h_t}(x_i) = \mathbf{p_i} = [p_1, p_2, \cdots p_{C}] \in \mathbb{R}^{C}$ is the probability over the classes output by the weak classifier for input $x_i$.

In step 6 of the algorithm, a weight $\alpha_t$ is chosen as the weight of weak classifier $\mathbf{h_t}$.  In \cite{SchapireSinger1999}, theoretical justification is given for choosing $\alpha_t = \frac{1}{2}\log (\frac{1+r_t}{1-r_t})$ where $r_t$ is the weighted average of the margin computed over the training set.  In our experiments, we found that the sigmoid function, $\alpha_t = \frac{1}{1+e^{-r_t}}$, worked better for both neural networks and decision trees.

Finally, the weights on the examples are updated according to the negative exponential of the margin for each example.  Then the whole process is repeated to train the next weak classifier.

\begin{algorithm}[thb]
\begin{algorithmic}[1]
\STATE Given: $(x_1, y_1)$, $(x_2, y_2), \cdots,$ $(x_m, y_m)$; $x_i \in \mathcal{X}$, $y_i \in \mathcal{Y}, \mathcal{Y} = \{1, ..., C\}$ where $C$ is the number of classes
\STATE Initialize $D_1(i) = \frac{1}{m}$

\FOR {$t=1,2,\cdots,T$}
\STATE Train weak classifier using distribution $D_t$ which yields weak classifier $\mathbf{h_t}: \mathcal{X}\rightarrow \mathbb{R}^C$ where $\mathbb{R}^C$ is a vector of $C$ reals representing a probability distribution over $C$ classes
\STATE Compute margins, $m_i$, for each training example $(x_i, y_i)$ according to equation \ref{eq:margin}
\STATE Choose $\alpha_t \in \mathbb{R}$ 
\STATE Update:
\vspace*{-12pt}
\begin{align*}
    D_{t+1}(i) &= \frac{D_t(i)e^{-\alpha_{t}m_{i}}}{Z_t}
\end{align*}
where $Z_t$ is a normalization factor (chosen to make $D_{t+1}$ a distribution).

\ENDFOR

\STATE Output the final hypothesis
\vspace*{-12pt}
\begin{equation*}
    H(x) = \text{argmax} (\sum_{t=1}^{T} \alpha_t \mathbf{h_t}(x))
\end{equation*}

\end{algorithmic}
\caption{A Multi-class Version of AdaBoost}
\label{alg:boost}
\end{algorithm}

\section{Boosted Decision Trees}
\label{sec:boosted_DTs}

\newtheorem{theorem}{Theorem}
\newtheorem{proof}{\bf Proof:}

Before examining boosted neural networks, we want to first point out some important properties of decision trees ensembles.  For our purposes, we will focus on binary decision trees.  A binary decision tree is defined as a binary tree with internal and leaf nodes.  It takes as input a feature vector or image and outputs a value (which could be a discrete class, a real value or a vector of real values).  Each internal node contains a function of the input which is compared to a threshold to determine whether control continues down the left or right branch.  Each leaf node contains an output value.  The number of leaves in a binary decision tree is equal to the number of internal nodes plus 1.  The number of parameters, $P$, in a decision tree with $I$ internal nodes and $L$ leaf nodes is
\begin{equation}
    P = k I + c L = k (L-1) + c L
\end{equation}
where $k$ is the number of parameters in the internal node function and $c$ is the length of a leaf's output vector.  Thus, the number of parameters in a decision tree is directly proportional to the number of leaves.  We will use the number of leaves for convenience as a substitute for the number of parameters.

One important observation is that the sum of two decision trees each with $P$ leaves cannot be represented by a single decision tree with $2 \times P$ leaves.  This is fundamentally due to the fact that summation is not an operation in a decision tree (in contrast to CNNs and MLPs where weighted sums, i.e. dot products, are a fundamental operation).

\begin{theorem}
The sum of two decision trees, each with $P$ leaves, cannot be represented by a single decision tree containing $2 \times P$ leaves when $P>2$.
\end{theorem}
\begin{proof}
By definition, a decision tree with $2 \times P$ leaves has $2 \times P$ possible outputs.  A sum of two decision trees, each with $P$ leaves and thus $P$ outputs, has $P \times P$ possible outputs.  The first tree could output $P$ possible values and each of these could be added to any of the $P$ possible values of the second tree, for a total of $P \times P$ possible outputs for the sum.  If $P>2$ then $P \times P > 2 \times P$.  Therefore, a decision tree with only $2 \times P$ possible outputs cannot represent the $P \times P$ possible outputs of the sum of two decision trees.
\end{proof}
This theorem generalizes straightforwardly to sums of N decision trees with P leaves which cannot be represented by a single decision tree with $N \times P$ leaves (when $P>1$, $N>2$).

The implication of this theorem is simply that the space of functions that a sum of decision trees can represent is different from the space of functions that a single decision tree with the same number of parameters can represent.  Summing decision tress gives you something you cannot get with a single decision tree.  However, this does not mean that a particular decision tree ensemble trained by AdaBoost (or any other algorithm) will necessarily be more accurate (in terms of training or testing error) than a particular single decision tree with the same number of total leaves.  It also does not imply that the set of classifiers representable by an ensemble of $N$ decision trees with $P$ leaves is a superset of the set of classifiers representable by single decision trees with $N \times P$ leaves.  There may be classifiers that are representable by a single decision tree with $N \times P$ leaves that are not representable by an ensemble of $N$ trees each with $P$ leaves.  Thus, the theorem does not prove that training or testing error for an ensemble of trees is always better than the error for a single tree with the same total number of leaves.  What the theorem does show is simply that the output space is larger for an ensemble of $N$ decision trees with $P$ leaves compared to a single decision tree with $N \times P$ leaves.
The important implication for our purposes is simply that an ensemble of decision trees is not equivalent to any single decision tree with the same number of total parameters. We will show in Section \ref{sec:boosted_CNNs} that this is in contrast to the case with ensembles of neural networks.


Another motivation for choosing decision tree ensembles over single, large decision trees is the fact that boosted ensembles of decision trees are much less susceptible to overfitting.  The fact that decision trees tend to overfit is well-known and well-established \cite{BreimanEtAl1984,Quinlan1993}.  A theoretical explanation for why boosted ensembles tend to be robust to overfitting, on the other hand, is still being debated.  Some proposals for a theoretical justification include the margin explanation of \cite{SchapireEtAl1998,ReyzinSchapire2006} and \cite{GaoZhou2013} and the statistical view of boosting of \cite{FriedmanEtAl2000}.
The important point here, however, is simply that there is a lot of empirical evidence showing that in many cases, boosted ensembles of decision trees avoid overfitting and generalize well even when large numbers of trees (and thus large numbers of parameters) are used \cite{ViolaJones2004,SchapireSinger1999,ViolaEtAl2003,BourdevBrandt2005} in contrast to the overfitting that is observed when very large decision trees are trained.  We will show in the next section that these same motivations do not apply to ensembles of CNNs or MLPs.

\section{Boosted CNNs}
\label{sec:boosted_CNNs}

\begin{figure*}[thb]
    \centering
    \includegraphics[width=\linewidth]{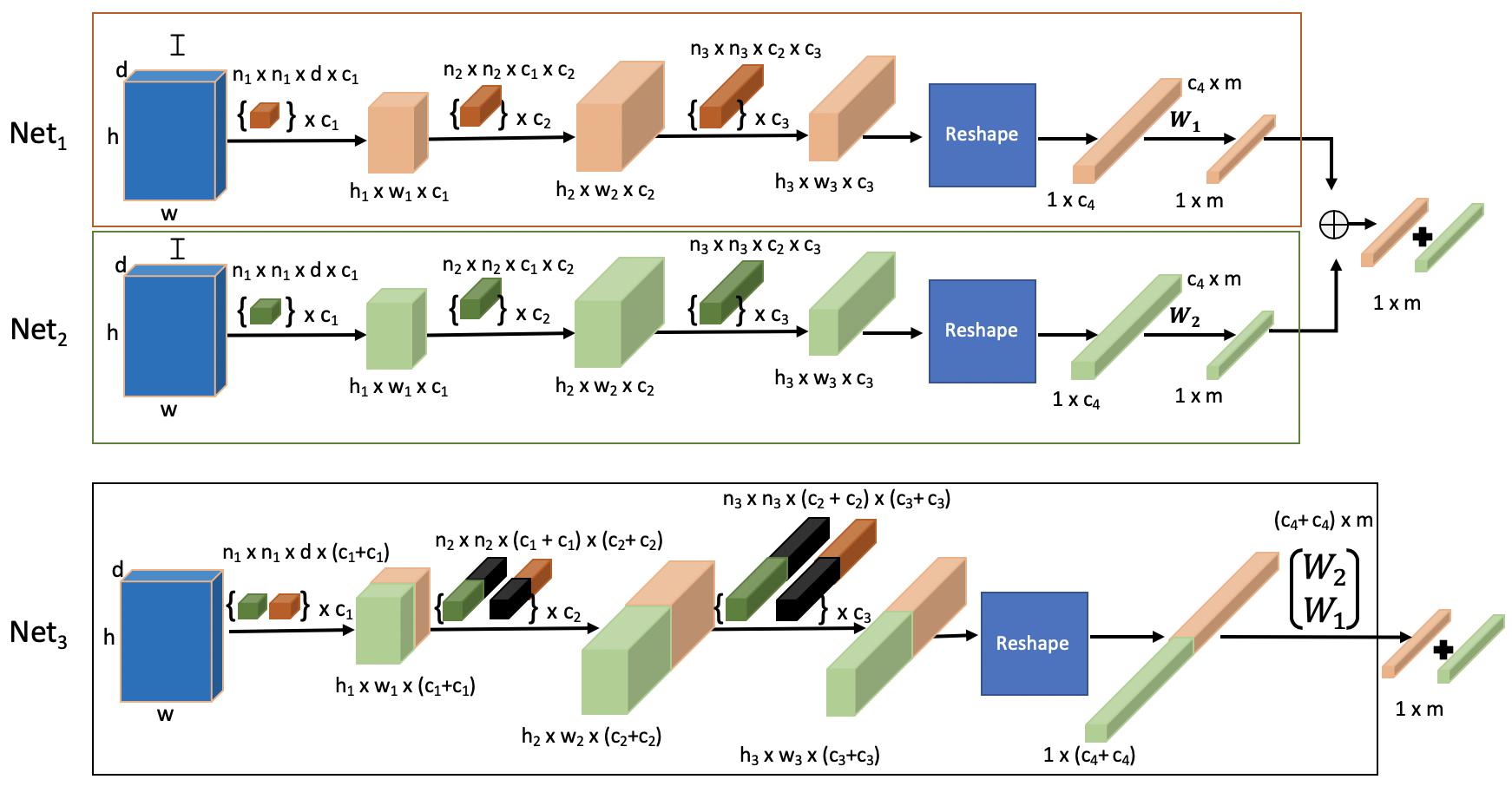}
    \caption{An illustration showing how to construct a single CNN with $2P$ parameters that is equivalent to the sum of two CNNs with $P$ parameters.  The arrows with cubes above them represent convolutional layers with the cubes representing the convolutional tensors.  The cubes at the beginning and end of the arrows represent feature maps (or the input image).  Black cubes represent zeros. The reshape block simply vectorizes the feature maps into a single 1-D vector.  $W_1$ and $W_2$ represent fully connected weight matrices.}
    \label{fig:const}
\end{figure*}

In contrast to boosted decision trees, we will prove in this section that boosted CNNs are not more expressive than a single CNN with the same number of total parameters.  

Consider two CNNs each containing P total parameters (weights), consisting of a number of convolutional layers with point-wise nonlinear activation functions as well as possibly containing fully connected layers at the end of the network.  For now, assume both CNNs have identical network architectures (but different parameter values).

We will show that the sum of the two CNNs can be exactly represented by a single CNN network with $2 \times P$ parameters.  The proof is by construction. Essentially, the single CNN simply runs both smaller CNNs in parallel and adds their results.  Fundamentally, this is possible because summation is a basic function of CNNs (in contrast to decision trees).

\begin{theorem}
A sum of 2 convolutional neural networks with $P$ parameters each can be represented by a single CNN with $2 \times P$ parameters.

\end{theorem}
\begin{proof}
\label{proof:CNN}

We prove the theorem by a simple construction.

Let {\em Net}$_1$ and {\em Net}$_2$ represent two CNNs with the same architecture containing $P$ parameters each.  We create a new network, {\em Net}$_3$, that is equal to {\em Net}$_1$ + {\em Net}$_2$ and contains $2 \times P$ parameters.  Figure \ref{fig:const} illustrates an example CNN with 3 convolutional layers (possibly including pooling layers) followed by a reshaping layer that creates a one-dimensional feature vector followed by a fully-connected layer.

For the first convolutional layer of {\em Net}$_3$, we simply concatenate the first $n_1 \times n_1 \times d \times c_1$ convolutional tensors of {\em Net}$_1$ and {\em Net}$_2$ to yield a $n_1 \times n_1 \times d \times (c_1 + c_1)$ convolutional tensor for {\em Net}$_3$.  The feature maps output by the first convolutional layer of {\em Net}$_3$ will be of dimension $h \times w \times (c_1 + c_1)$ and will be the feature maps of the first convolutional layers of {\em Net}$_1$ and {\em Net}$_2$ concatenated (each of dimension $h \times w \times c_1$).

For the second and later convolutional layers of {\em Net}$_3$, we take the $n_2 \times n_2 \times c_1 \times c_2$ convolutional tensor of {\em Net}$_1$ and pad the third dimension of the tensor with zeros at the end to get an $n_2 \times n_2 \times (c_1 + c_1) \times c_2$ tensor.  We do the same with the $n_2 \times n_2 \times c_1 \times c_2$ tensor of {\em Net}$_2$, except we add the zeros to the beginning of the third dimension to get another $n_2 \times n_2 \times (c_1 + c_1) \times c_2$ tensor.  Then we concatenate these two tensors to get an $n_2 \times n_2 \times (c_1 + c_1) \times (c_2 + c_2)$ tensor for {\em Net}$_3$ with the same number of parameters as the corresponding layers of {\em Net}$_1$ and {\em Net}$_2$.  To make it clear that the zeros in the tensors of {\em Net}$_3$ are not trainable parameters, all zeros signify the lack of a connection between neurons in {\em Net}$_3$.

For the fully connected layer, the weight matrices $W_1$, $W_2$ are concatenated along the rows to form the fully connected layer of {\em Net}$_3$.

\end{proof}

While the proof is shown for two networks with the same architecture, it can easily be extended to heterogeneous architectures.  The construction also generalizes to sums of many networks.  Furthermore, the construction for the fully connected layer shows that this construction would apply to a network consisting solely of fully-connected layers, i.e. an MLP.

This proof shows that there are no functions that can be represented by a CNN (or MLP) ensemble that cannot also be represented by a single CNN (or MLP) with the same number of parameters.  In Section \ref{sec:boosted_DTs}, we found the opposite to be true in the case of decision tree ensembles.  However, this proof does not mean that a particular CNN ensemble trained by AdaBoost or some other algorithm cannot be more accurate than a particular single CNN with the same number of parameters.  This is an empirical question which we explore through various experiments in Section \ref{sec:CNN_experiments}.  Proof \ref{proof:CNN} simply tells us that the architecture of CNN ensembles is not adding anything that we cannot already get with a single CNN.

The other reason that we previously discussed for the success of decision tree ensembles is that they tend to avoid the tendency to overfit that is observed in single, large decision trees.  Once again, this motivation is not very compelling for ensembles of CNNs because many CNN architectures already exhibit a surprising robustness to overfitting despite being massively overparameterized.  The computer vision literature is full of papers showing good generalization accuracy using overparameterized deep networks on many different datasets.  While a theoretical understanding of this behavior is still an active research topic, there have been some attempts at an understanding, including the works of \cite{BelkinEtAl2019}, \cite{PoggioEtAl2018}, and \cite{ZhangEtAl2017}.


\section{Decision Tree Experiments}
\label{sec:DTexp}

In this section we present some experiments on image classification datasets comparing boosted ensembles of decision trees to single decision trees with the same number of total leaves.  There have been very many papers published that plot the error rate of a boosted ensemble of decision trees as a function of the number of weak classifiers.  The vast majority of these plots show the error rate continuing to decrease as more weak classifiers are added to the ensemble.  We also examine the error rates of single decision trees as the number of leaves increases.  A comparison of the error rates of boosted decision tree ensembles and single decision trees with the same number of leaves is interesting here because it confirms that single decision trees often overfit as the number of leaves increases while boosted decision tree ensembles often do not, and as a consequence that boosted ensembles generalize much better than single decision trees.  We will show in Section~\ref{sec:CNN_experiments} that boosted neural network ensembles exhibit very different behavior.

We use binary decision trees with multi-class outputs.  The output is a probability distribution over all possible output classes.  Each node computes a Haar-like filter on the input image and thresholds the filter value to determine whether the input goes to the left or right child node.  When a leaf node is reached, the output class is the class with the highest probability.  The probability distribution for any leaf node is computed from the training examples that fall into that leaf as the count of each class divided by the number of training examples in the leaf.

We use a set of Haar-like filters that are very similar to the Haar-like filters used in \cite{ViolaJones2004}.  We use 2-rectangle, 3-rectangle and 4-rectangle Haar-like filters as illustrated in Figure \ref{fig:haarfilters}.  The value of a Haar-like filter applied to an image is the sum of the pixel values (in a single color channel) in the white rectangles minus the sum of the pixel values in the gray rectangles.

For the CIFAR-10 and CIFAR-100 datasets in which each example is 32x32 pixels, we used 2912 Haar-like filters sampled from all possible 2-rectangle, 3-rectangle and 4-rectangle filters that can fit within a 32x32 pixel image.  The minimum rectangle size within any filter was 5 pixels.  Because a filter can be applied to any of the three color channels independently, there were effectively $3 \times 2912 = 8736$ filters used for training.

For the MNIST dataset with 28x28 pixel examples, we used 3512 Haar-like filters with a minimum size of 4 pixels for any single rectangle within a filter.  There is only a single gray-level channel for MNIST.

\begin{figure*}[t]
\centering
\includegraphics[width=0.9\textwidth,keepaspectratio]{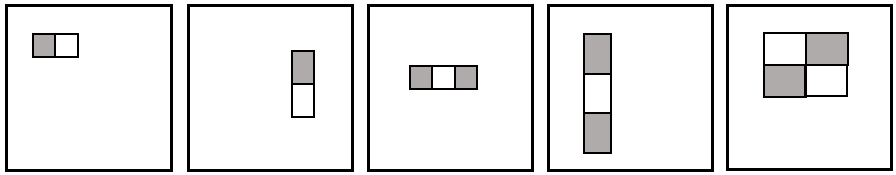}
\caption{Examples of 2, 3 and 4 rectangle Haar-like filters.  The sum of pixel values in shaded rectangles are subtracted from the sum of pixel values in white rectangles.  Each Haar-like filter has two versions: one with absolute value and one without.  A Haar-like filter can be applied to any color channel in the case of color input images (CIFAR-10 and CIFAR-100).}
\label{fig:haarfilters}
\end{figure*}

We use the following decision tree learning algorithm to build a tree with a fixed number of leaf nodes.  Initially, a leaf node is created containing all of the training examples.  On each iteration of the learning algorithm, the leaf node with the minimum "peak" is chosen to split where "peak" is defined as the probability of the highest probability class minus the average probability of all other classes.  The peak value is similar to negative entropy.  To split a node, all Haar-like filters in the given filter set are used to find the optimal filter and threshold that splits the examples in the node so that the sum of the peaks for the examples going into the children nodes is maximized. The Haar-like filter and threshold that maximize the sum of peaks of the children nodes is chosen as the decision function for the node.  Splitting of nodes continues until the desired number of leaves is reached.

We use the multi-class boosting algorithm described in Section \ref{sec:Adaboost}.  We conducted experiments on three standard image classification datasets: MNIST, CIFAR-10 and CIFAR-100.

\textbf{MNIST}: The MNIST dataset \cite{Lecun98gradient-basedlearning} is a handwritten digit (from 0 to 9) recognition task consisting of $28 \times 28$ pixel gray scale images.  The dataset contains 60,000 training images (6000 images of each digit) and 10,000 testing images (1000 images of each digit).

\textbf{CIFAR-10}: The CIFAR-10 dataset \cite{Krizhevsky09learningmultiple}  represents an object recognition task.  It contains 60,000 $32\times 32$ color images with each image containing one of 10 different classes. The 10 different classes represent airplanes, cars, birds, cats, deer, dogs, frogs, horses, ships, and trucks. There are 6,000 images of each class from which 5,000 are used for training and 1,000 are used for testing.

\textbf{CIFAR-100}: This dataset \cite{Krizhevsky09learningmultiple} represents a similar object recognition task as the CIFAR-10 dataset, except there are 100 object classes instead of 10.  Each class contains 600 color images.
Each image is $32\times32\times3$, and the 600 images are divided into 500 training, and 100 test for each class.

For each dataset, we trained boosted decision trees for 20 rounds with a fixed number of leaves (64 for MNIST and CIFAR-10, 128 for CIFAR-100) in each weak classifier. We also trained a series of single decision trees with the same number of total leaves as in each ensemble.  To train a single decision tree with 128 leaves, for example, we start with one already trained with 64 leaves and continue the decision tree training algorithm until 128 leaves are reached.  Since there is no randomness in the training algorithm, each experiment is run only once. (Rerunning the same experiment would yield the exact same classifier.)

\begin{figure*}[!ht]   
\centering
\subfloat[]{\includegraphics[width=0.33\textwidth, keepaspectratio]{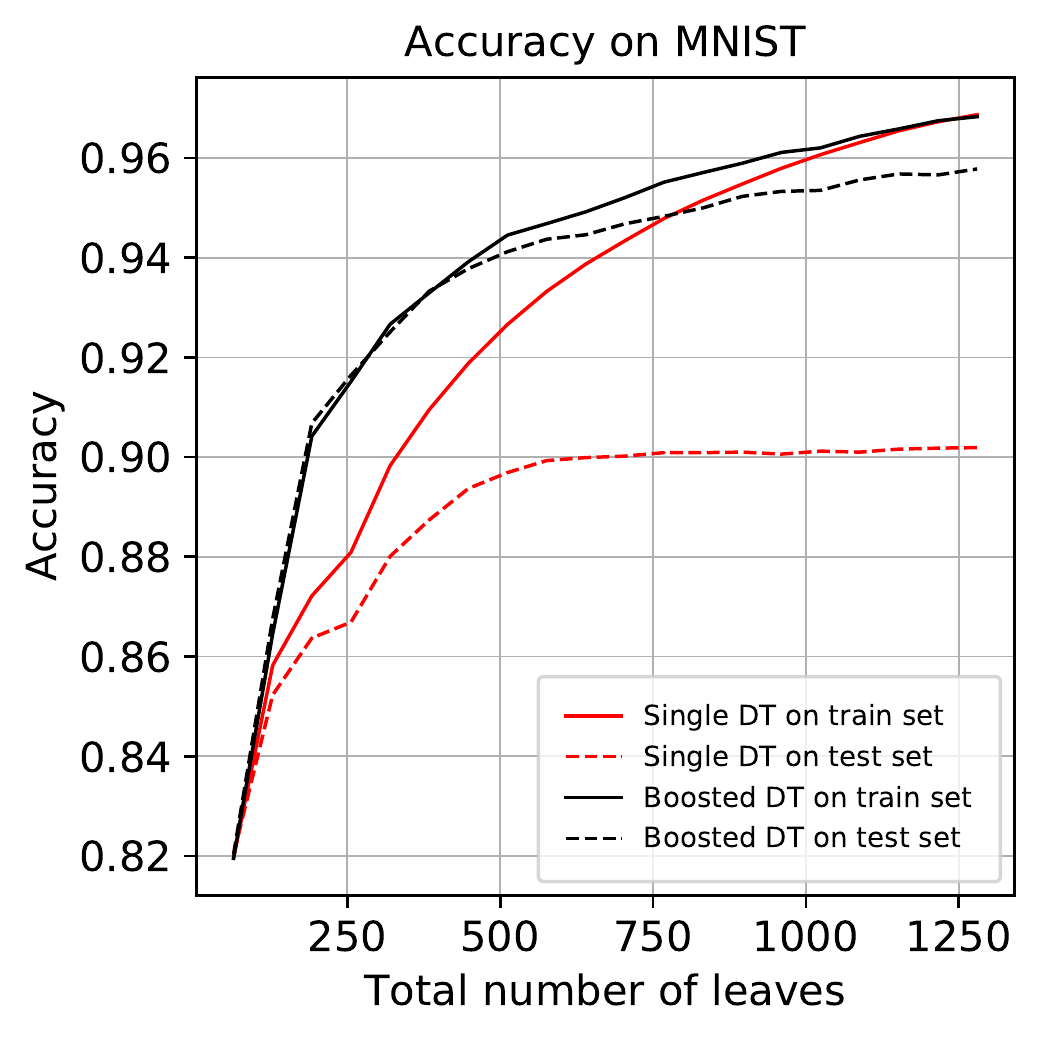}\label{fig:mnist_dt}}
\subfloat[]{\includegraphics[width=0.33\textwidth, keepaspectratio]{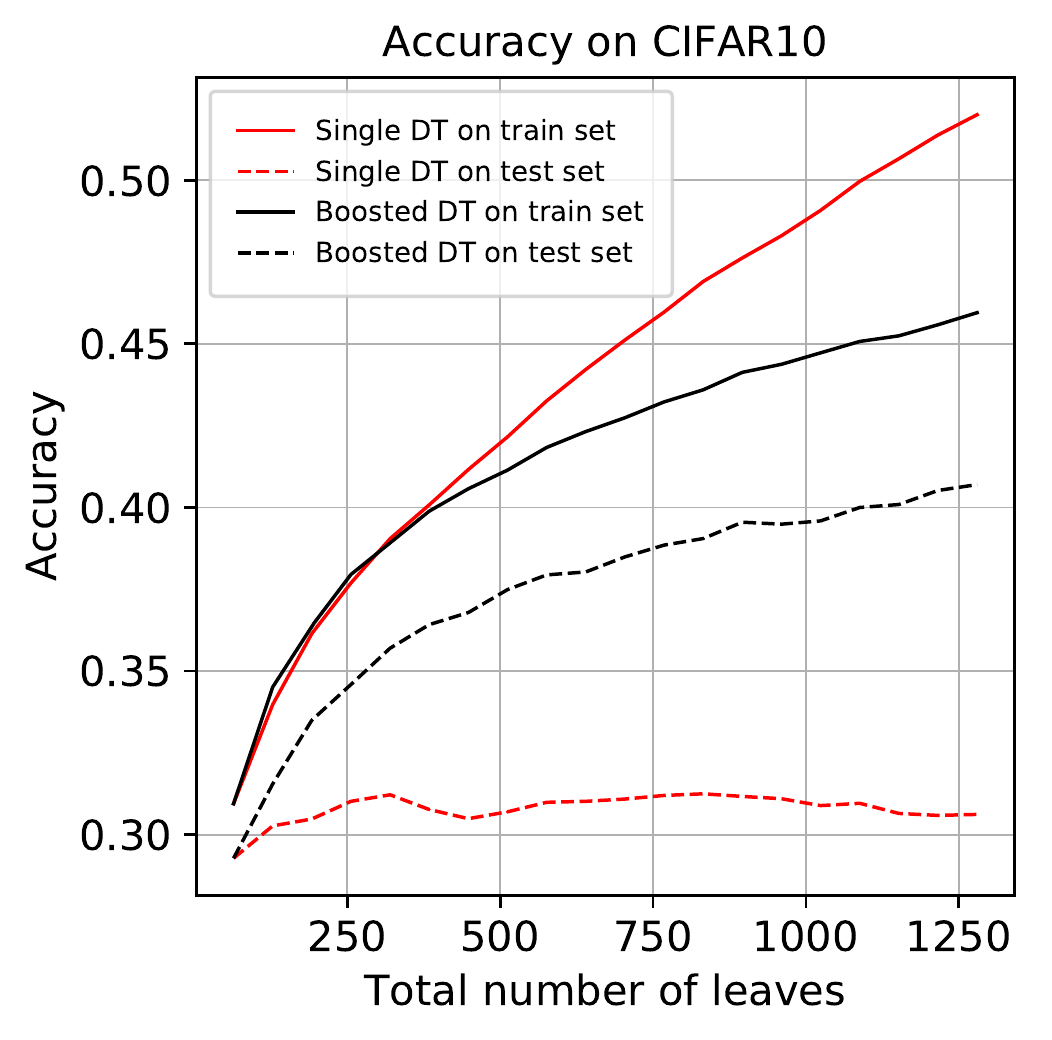}\label{fig:cifar10_dt}}
\subfloat[]{\includegraphics[width=0.34\textwidth, keepaspectratio]{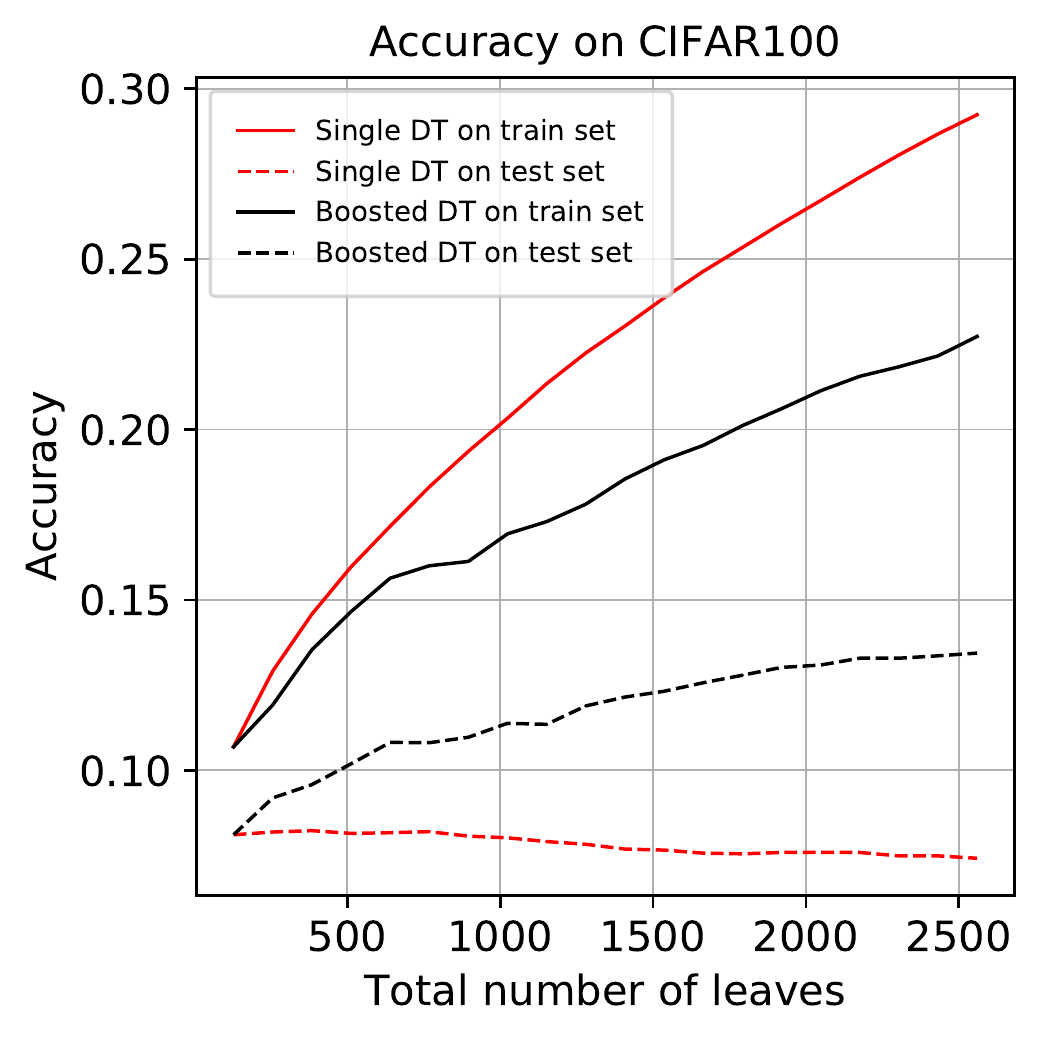}\label{fig:cifar100_dt}}
\caption[Optional caption for list of figures 5-8]{Classification accuracy on MNIST [\ref{fig:mnist_dt}], CIFAR-10 [\ref{fig:cifar10_dt}] and CIFAR-100 [\ref{fig:cifar100_dt}] training and testing sets for single decision tree and boosted decision tree versus total number of leaves in the tree/ensemble.}
\label{fig:dt}
\end{figure*}

In each of the Figures \ref{fig:mnist_dt}-\ref{fig:cifar100_dt}, the red curves are for single decision trees and the black curves are for boosted decision trees.  The solid curves show accuracy on the training set and the dashed curves show accuracy on the testing set.

Figure \ref{fig:mnist_dt} shows that on the MNIST dataset, the boosted decision trees are significantly more accurate on the testing set than the single decision trees with the same number of leaves.  On the training set, the boosted and single trees are more comparable, with the boosted trees being a little more accurate until the number of leaves becomes greater than 1000 or so.

On CIFAR-10 (Figure \ref{fig:cifar10_dt}) and CIFAR-100 (Figure \ref{fig:cifar100_dt}), we again see that the boosted decision trees are significantly more accurate on the testing sets, but not on the training sets as the number of leaves grows.  The fact that the single decision trees continue to improve accuracy on the training set while worsening accuracy on the testing sets is clear evidence of overfitting.  The boosted decision trees, on the other hand, avoid overfitting as both training and testing accuracy continue to increase with increasing numbers decision trees.

These experiments confirm that boosted decision trees usually generalize better than single decision trees with the same number of total leaves.  They also confirm that training single large decision trees is prone to overfitting while boosted ensembles of decision trees are resistant to overfitting.

\section{Neural Network Experiments}
\label{sec:CNN_experiments}


In this section we explore empirically whether boosted neural networks are similar to boosted decision trees and result in better accuracy than simply using a single network with an equivalent number of parameters.  We experiment with three network architectures described below and three datasets.  As with the experiments using decision trees, we use the CIFAR-10 and CIFAR-100 object recogniton datasets.  Because even the small base networks we use achieve nearly 100\% training accuracy on MNIST, we instead use the Street View House Numbers (SVHN) dataset for our third dataset.  SVHN \cite{37648} is a real-world image dataset obtained from house numbers in Google Street View images.  The dataset contains over $600,000$ training images, and about $20,000$ test images, each of size $32 \times 32$ pixels.

\begin{table}[h]
  \centering
  \scriptsize
  \renewcommand{\arraystretch}{1.0}
  \renewcommand{\tabcolsep}{1.4mm}
  \resizebox{\linewidth}{!}{
  \begin{tabular}{c|cc||c|cc}
  \toprule
  \multicolumn{3}{c||}{\textbf{SVHN \& CIFAR-10}} & \multicolumn{3}{c}{\textbf{CIFAR-100}}\\ \hline
  \textbf{Ensemble} & \multicolumn{2}{c||}{\textbf{Single CNN}} & \textbf{Ensemble} &  \multicolumn{2}{c}{\textbf{Single CNN}} \\ 
   \textbf{No. of Parameters}& \textbf{Hidden Dimensions} & \textbf{No. of Parameters} & \textbf{No. of Parameters} &\textbf{Hidden Dimensions} & \textbf{No. of Parameters}\\
    \midrule
    $5954$ & $[6, 16, 32]$ & $5954$ & $8834$ & $[6, 16, 32]$& $8834$\\
    $11908$ & $[12, 16, 64]$ & $11908$ & $17668$ & $[12, 16, 64]$& $17668$\\
    $17862$ & $[18, 16, 96]$ & $17862$& $26502$ & $[18, 16, 96]$& $26502$\\
    $23816$ & $[24, 16, 128]$ & $23816$& $35336$ & $[24, 16, 128]$& $35336$\\
    $29770$ & $[30, 16, 160]$ & $29770$& $44170$ & $[30, 16, 160]$& $44170$\\
    $35724$ & $[14, 39, 82]$ & $34894$& $53004$ & $[14, 39, 105]$& $52647$\\
    $41678$ & $[15, 42, 88]$ & $40219$& $61838$ & $[15, 42, 114]$& $60567$\\
    $47632$ & $[16, 45, 95]$ & $46337$& $70672$ & $[16, 45, 122]$& $68522$\\
    $53586$ & $[18, 48, 100]$ & $52462$& $79506$ & $[18, 48, 129]$& $76890$\\
    $59540$ & $[18, 50, 106]$ & $57346$& $88340$ & $[18, 50, 136]$& $83386$\\
  \bottomrule
  \end{tabular}
  }
  \caption{Table enumerating hidden dimensions and number of trainable parameters of the ensemble and single CNN classifier on CIFAR-10/SVHN and CIFAR-100.}
  \label{tab:cnn}
\end{table}

\begin{figure}
    \centering
    \includegraphics[width=\linewidth]{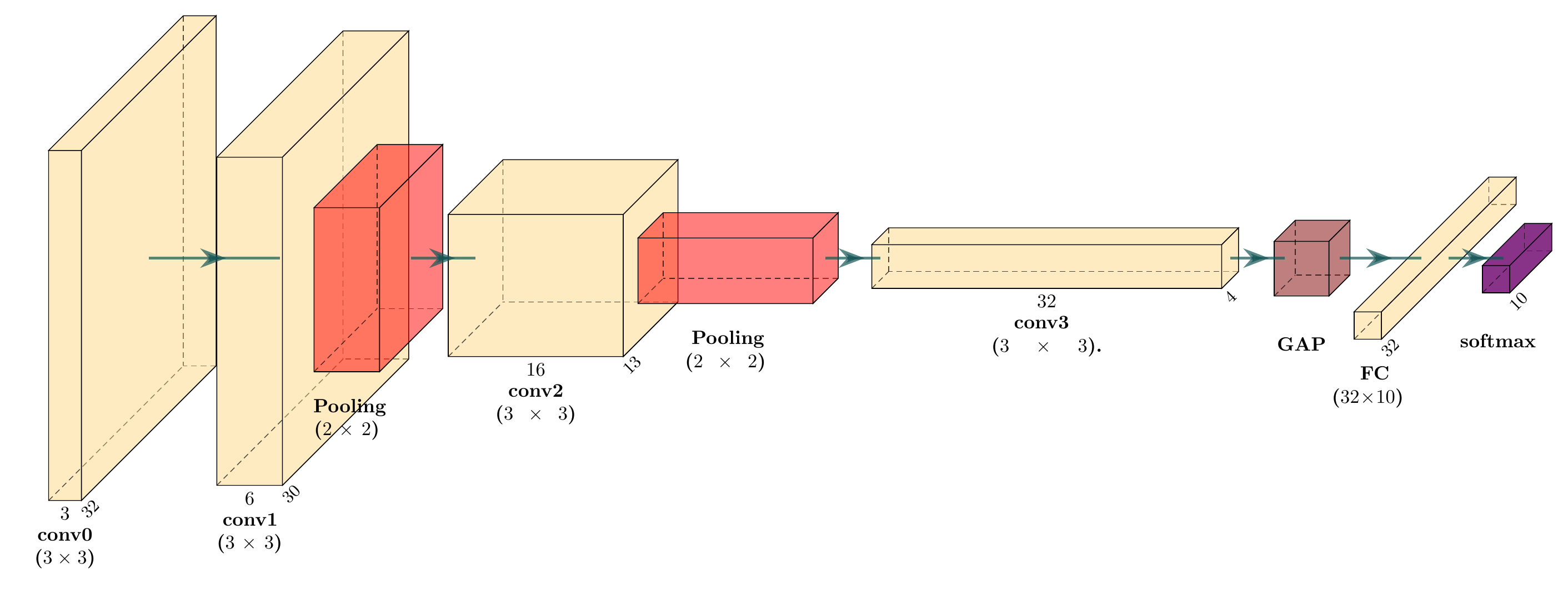}
    \caption{CNN architecture for CIFAR-10/SVHN: The network consists of three convolution layers with $3\times 3$ filters, $0$ padding and stride $1$. The convolution layers are followed by a ReLU non-linearity. We use max pooling in this work with a filter size of $2\times 2$, no padding and stride $2$ which results in a downsampling of the features by a factor of 2. The three convolution layers have 6, 16 and 32 filters respectively. Finally, a Global Average Pooling (GAP) is applied and a fully connected (fc) outputs logits over the number of classes.}
    \label{fig:cnn_arch}
\end{figure}

\textbf{CNN}: The first architecture is a CNN base classifier. We adopt a similar architecture as LeNet \cite{Lecun98gradient-basedlearning}, as a base learner, for our experiments and is shown in Fig. \ref{fig:cnn_arch}.
Boosting such a model for $N$ rounds makes the total number of parameters $N$ times the number of parameters of the base classifier. For training a single model with an equal number of parameters, we increase the number of filters in each hidden layer until the total number of parameters is less than or equal to the required number.  To increase the number of parameters in the CNN for creating larger single networks, we add convolutions (width) to each existing layer but do not add new layers (depth). 

For a single CNN for which we want $n$ times more parameters than the base CNN, 
one way to increase the number of parameters, for a network with odd number of layers (which is true in our case), is to multiply the number of convolutional filters in every other layer by $n$. This increases the number of parameters by exactly $n$. However for higher values of $n$, this method results in a tight bottleneck in the intermediate layers. To alleviate this issue, we use this method only for smaller values of $n$ ($n \in \{1,\cdots,5\}$) and for larger values of $n$, we 
multiply the number of filters in each layer of the base CNN by $\sqrt{n}$ (rounding to the nearest integer).  Multiplying the number of filters in layers $L$ and $L+1$ by $\sqrt{n}$, multiplies the number of {\em parameters} in layer $L+1$ by $n$ (except for the first and last layers).  Since the number of parameters in the first and last layers only increases by a factor of $\sqrt{n}$, we multiply the number of neurons in the penultimate layer by a factor $c \sqrt{n}$ to make the total number of parameters about the same as in the boosted counterpart.  We use $c=1.35$ in this work for CNN base classifiers.  It is likely that the architectures could be optimized to produce even higher accuracy in our experiments, but this is not the point of the current work. The CNN base classifier has $5954$ trainable parameters for CIFAR-10, SVHN ($c=1.05$) and $8834$ trainable parameters for CIFAR-100 ($c=1.35$). Refer to Table \ref{tab:cnn} for exact details of the architecture.

\begin{table}[h]
  \centering
  \scriptsize
  \renewcommand{\arraystretch}{1.0}
  \renewcommand{\tabcolsep}{1.4mm}
  \resizebox{\linewidth}{!}{
  \begin{tabular}{c|cc||c|cc}
  \toprule
  \multicolumn{3}{c||}{\textbf{SVHN \& CIFAR-10}} & \multicolumn{3}{c}{\textbf{CIFAR-100}}\\ \hline
  \textbf{Ensemble} & \multicolumn{2}{c||}{\textbf{Single MLP}} & \textbf{Ensemble} &  \multicolumn{2}{c}{\textbf{Single MLP}} \\ 
   \textbf{No. of Parameters}& \textbf{Hidden Dimensions} & \textbf{No. of Parameters} & \textbf{No. of Parameters} &\textbf{Hidden Dimensions} & \textbf{No. of Parameters}\\
    \midrule
    $410880$ & $[128, 128]$ & $410880$ & $422400$ & $[128, 128]$& $422400$\\
    $821760$ & $[246,246]$ & $818688$ & $844800$ & $[247, 247]$& $844493$\\
    $1232640$ & $[358, 358]$ & $1231520$& $1267200$ & $[358, 358]$& $1263740$\\
    $1643520$ & $[463, 463]$ & $1641335$& $1689600$ & $[464, 464]$& $1687104$\\
    $2054400$ & $[563,  563]$ & $2052135$& $2112000$ & $[565, 565]$& $2111405$\\
    $2465280$ & $[658, 658]$ & $2460920$& $2534400$ & $[661, 661]$& $2533613$\\
    $2876160$ & $[750, 750]$ & $2874000$& $2956800$ & $[753, 753]$& $2955525$\\
    $3287040$ & $[838, 838]$ & $3284960$& $3379200$ & $[841, 841]$& $3374933$\\
    $3697920$ & $[923, 923]$ & $3696615$& $3801600$ & $[927, 927]$& $3799773$\\
    $4108800$ & $[1005, 1005]$ & $4107435$ & $4224000$ & $[1010, 1010]$& $4223820$\\

  \bottomrule
  \end{tabular}
  }
  \caption{Table enumerating hidden dimensions and number of trainable parameters of the ensemble and single MLP classifier on CIFAR-10/SVHN and CIFAR-100.}
  \label{tab:mlp}
\end{table}
\begin{figure}
    \centering
    \includegraphics[height=0.5\linewidth, width=0.6\linewidth]{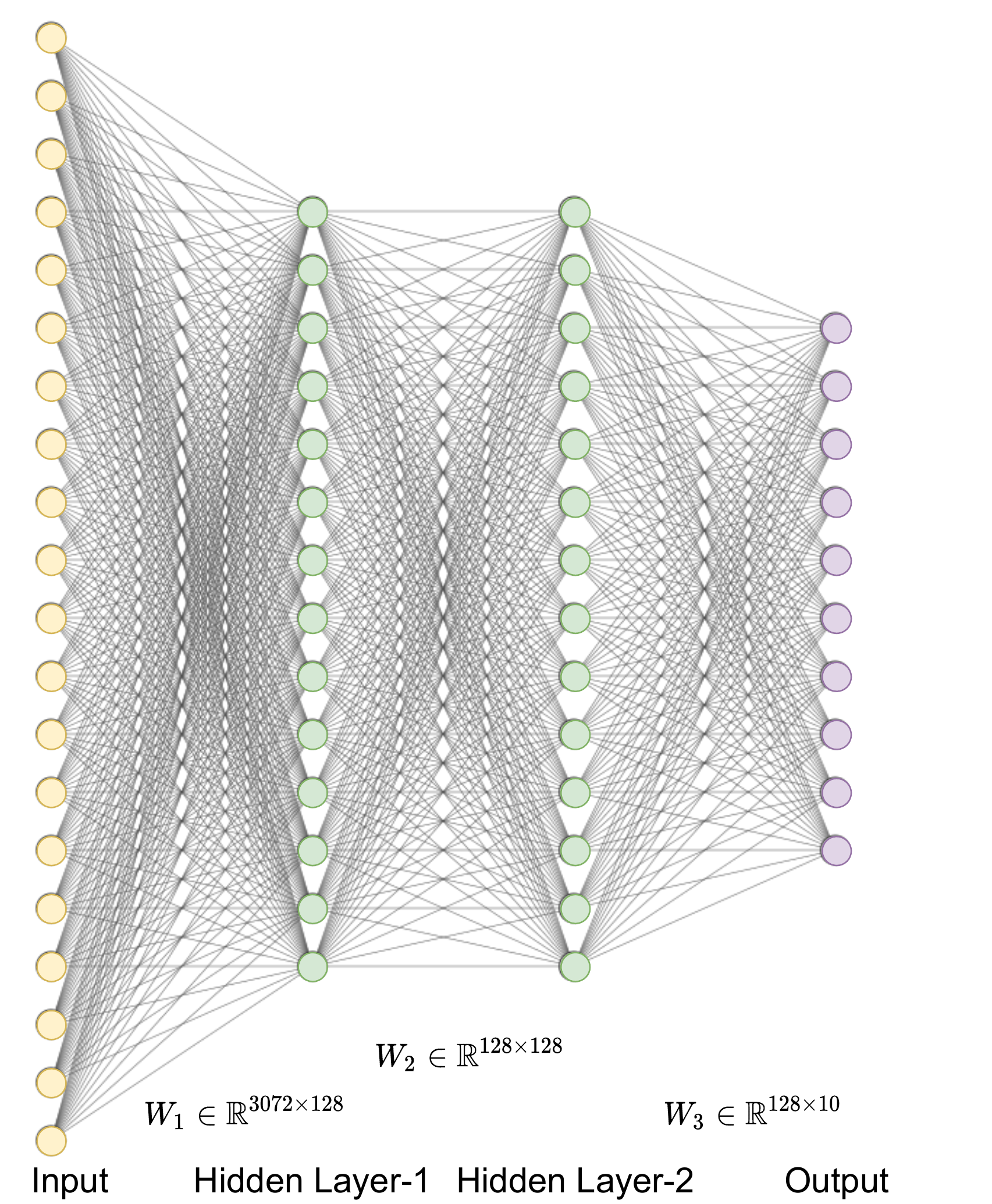}
    \caption{MLP architecture for CIFAR-10/SVHN: We adopt a four-layer MLP with two hidden layers each of dimension $128$. The input layer has $3072$ units. We use ReLU non-linearity after each layer except the final one. }
    \label{fig:mlp_arch}
\end{figure}
\textbf{MLP}: We adopt a four-layer MLP with two hidden layers as shown in Fig. \ref{fig:mlp_arch} as the base learner.
To increase the number of parameters in the MLP to approximately match the number of parameters in each boosted ensemble of MLPs, we assume the same number, $n$, of neurons in both the MLP hidden layers and solve for $n$ analytically.  The total number of parameters in the MLP is $p(n) = 3072*n + n^2 + Cn = n^2 + (3072 + C)*n$, where $C$ is the number of output neurons. We do not consider the bias terms in our experiments. To obtain a network with $N$ parameters, we solve the quadratic equation $p(n) = N$ and round the solution to the lowest integer. 
Each MLP base classifier has $410880$ trainable parameters for CIFAR-10 and SVHN and $422400$ trainable parameters for CIFAR-100 (due to more output classes). Refer to Table \ref{tab:mlp} for exact details of the architecture and number of trainable parameters.

\begin{figure}
    \centering
    \includegraphics[width=\linewidth]{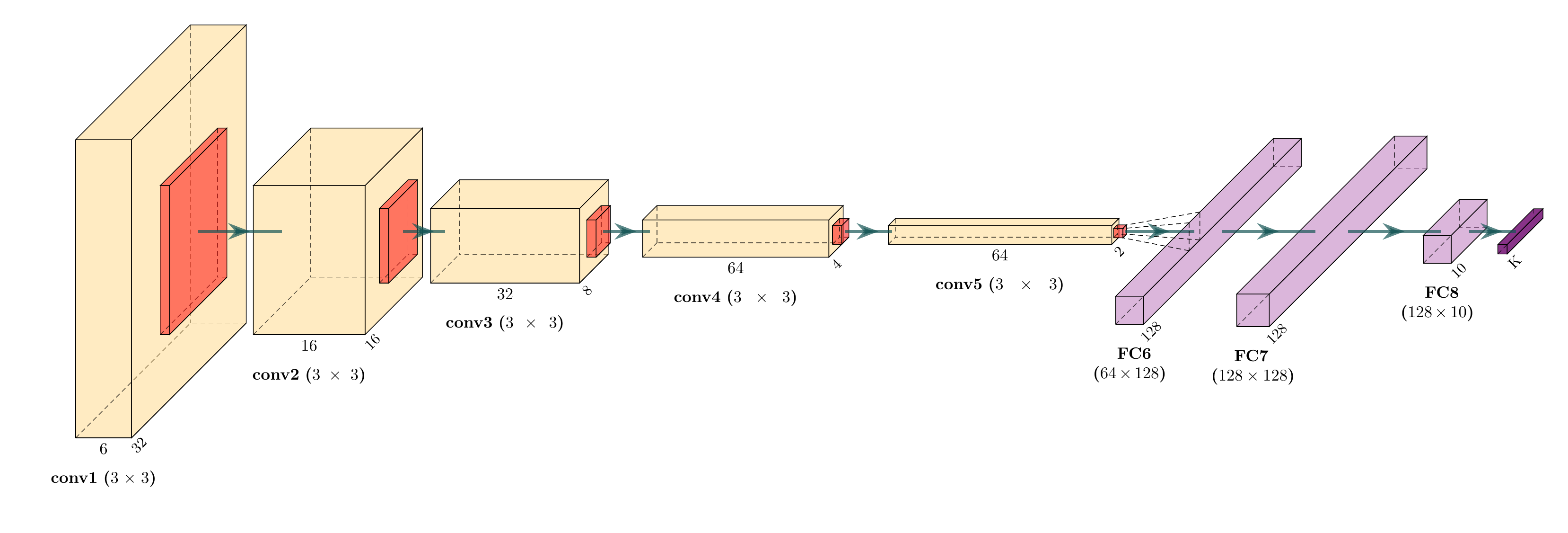}
    \caption{VGG-8 architecture for CIFAR-10/SVHN: We implement a non-standard VGG-8 architecture (without Batch Normalization), which is similar to the standard VGG-11 \cite{Simonyan15} network but with five convolution layers instead of eight. Each convolution layer has $3\times 3$ filters with padding $1$ and stride $1$. The convolution layer is followed by ReLU non-linearity and a max pooling layer with $2\times 2$ filter, no padding and stride 2 resulting in a downsampling factor of 2. The convolution layers are followed by 3 fully connected layers of dimensions 128, 128 and 10. We use ReLU non-linearity after each fc layer except the final one. }
    \label{fig:vgg_arch}
\end{figure}
\textbf{VGG-8}: To demonstrate results using a deeper architecture, we adopt a VGG \cite{Simonyan15} style architecture as the base learner. 
We follow a similar strategy to increase the number of parameters as with the CNN mentioned above. VGG-8 classifier has $87234$ and $98844$ trainable parameters for CIFAR-10 ($c=1.48$) and CIFAR-100 ($c=1.70$), respectively. We show the exact details of the architecture in Table \ref{tab:vgg}.

\begin{table}[h]
  \centering
  \scriptsize
  \renewcommand{\arraystretch}{1.0}
  \renewcommand{\tabcolsep}{1.4mm}
  \resizebox{\linewidth}{!}{
  \begin{tabular}{c|cc||c|cc}
  \toprule
  \multicolumn{3}{c||}{\textbf{CIFAR-10}} & \multicolumn{3}{c}{\textbf{CIFAR-100}}\\ \hline
  \textbf{Ensemble} & \multicolumn{2}{c||}{\textbf{Single VGG-8}} & \textbf{Ensemble} &  \multicolumn{2}{c}{\textbf{Single VGG-8}} \\ 
   \textbf{No. of Parameters}& \textbf{Hidden Dimensions} & \textbf{No. of Parameters} & \textbf{No. of Parameters} &\textbf{Hidden Dimensions} & \textbf{No. of Parameters}\\
    \midrule
    $87234$ & $[6, 16, 32, 64, 64]$ & $87234$ & $98844$ & $[6, 16, 32, 64, 64]$& $98844$\\
    $174648$ & $[8, 22, 45, 90, 133]$ & $190142$ & $197688$ & $[8, 22, 45, 90, 153]$& $220532$\\
    $261972$ & $[10, 27, 55, 110, 164]$ & $272163$& $296532$ & $[10, 27, 55, 110, 188]$& $310629$\\
    $349296$ & $[12, 32, 64, 128, 189]$ & $356215$& $395376$ & $[12, 32, 64, 128, 217]$& $403693$\\
    $436620$ & $[13, 35, 71, 143, 211]$ & $435156$& $494220$ & $[13, 35, 71, 143, 243]$& $492078$\\
    $523944$ & $[14, 39, 78, 156, 232]$ & $516055$& $593064$ & $[14, 39, 78, 156, 266]$& $579787$\\
    $611268$ & $[15, 42, 84, 169, 250]$ & $596331$& $691908$ & $[15, 42, 84, 169, 287]$& $668991$\\
    $698592$ & $[16, 45, 90, 181, 267]$ & $677620$& $790752$ & $[16, 45, 90, 181, 307]$& $759550$\\
    $785916$ & $[18, 48, 96, 192, 284]$ & $761294$& $889596$ & $[18, 48, 96, 192, 326]$& $850898$\\
    $873240$ & $[18, 50, 101, 202, 299]$ & $838108$& $988440$ & $[18, 50, 101, 202, 344]$& $937333$\\

  \bottomrule
  \end{tabular}
  }
  \caption{Table enumerating hidden dimensions and number of trainable parameters of the ensemble and single VGG-8 classifier on CIFAR-10/SVHN and CIFAR-100.}
  \label{tab:vgg}
\end{table}

\textbf{Network Training Parameters}: For all experiments using the CNN base classifier, we use a batch size of 128 with an initial learning rate of $0.1$ ($0.0001$) with SGD (ADAM) optimizer trained for $300$ epochs. The learning rate is decreased by a factor of ten after $95$ epochs for each optimizer. We do not perform any hyperparameter optimization. We boost the base classifier for ten rounds. 
We report results using two optimizers (SGD and Adam) for the CNN classifier and only SGD (which yields better accuracy) for MLP and VGG networks. The MLP classifiers are trained with a batch size of $128$ for $300$ epochs with an initial learning rate of $0.01$ using an SGD optimizer. The learning rate is decreased by a factor of $10$ after $95$ epochs. VGG-8 classifiers are trained for $200$ epochs with a batch size of $256$ using an SGD optimizer. We begin training with a learning rate of $0.1$ and decrease by a factor of $5$ after $60$, $120$ and $160$ epochs. We repeat all experiments in Figure \ref{fig:baby}, \ref{fig:mlp} and \ref{fig:vgg8} for 5 rounds and plot the mean and standard deviation.



\begin{table*}[!htb]
  \begin{minipage}[t]{0.59\linewidth}
    \includegraphics[width=\linewidth]{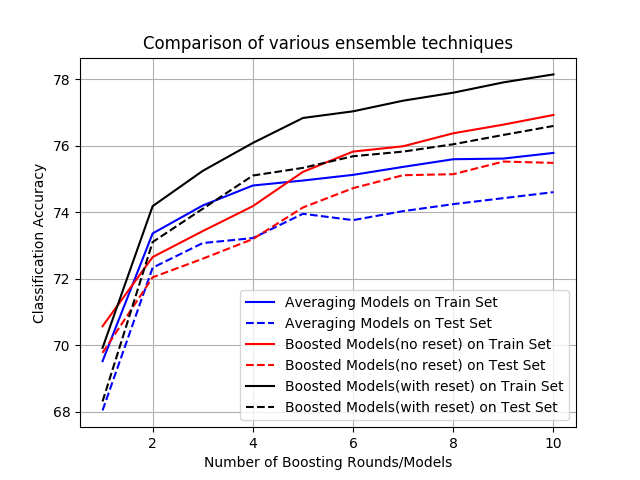}
    \captionsetup{width=.8\linewidth}
    \captionof{figure}{Comparison of averaging and boosting with/without reset of a CNN base classifier on CIFAR-10 train and test sets}
    \label{fig:ensemble}
    \end{minipage}\hfill
    \begin{minipage}[t]{0.41\linewidth}
    \includegraphics[width=\linewidth]{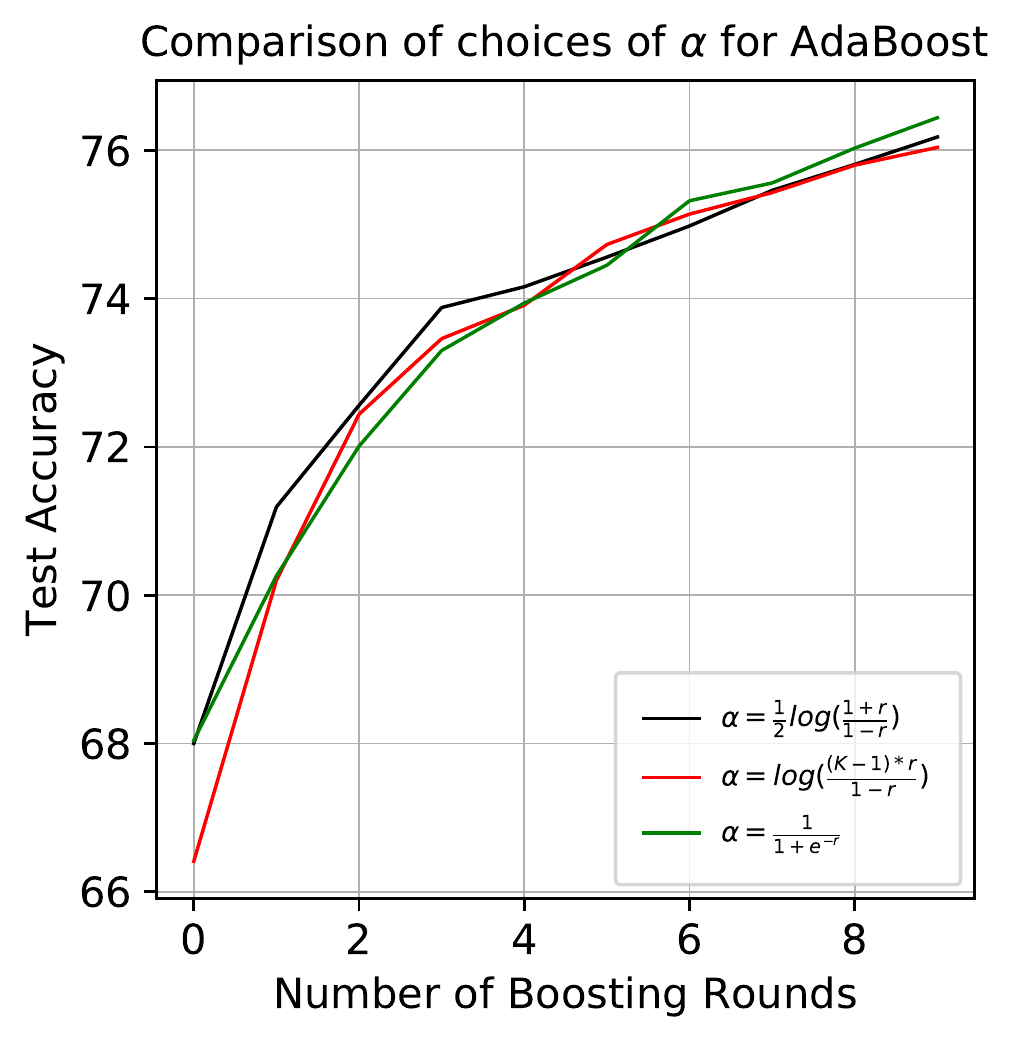}
        \captionsetup{width=.9\linewidth}
    \captionof{figure}{Comparison of different choices of $\alpha$ for AdaBoost on CIFAR-10 test set.}
        \label{fig:alpha_comp}
  \end{minipage}
  \end{table*}



\textbf{Boosting vs Averaging Models: }We start by comparing training and testing set performance of averaging and boosting with/without reset on the CIFAR-10 dataset. Boosting with reset, similar to the one in \cite{MoghimiEtAl2016}, means initializing the models in every round of boosting from scratch. Boosting without reset initializes the model with the weights of the first model. From Figure \ref{fig:ensemble}, it is clear that boosting with reset is most accurate, and that boosting offers improvement over naive averaging of models which is often used as a quick and simple way of improving accuracy.

\textbf{Choice of $\alpha$: } \cite{SchapireSinger1999} prove that choosing $\alpha = \frac{1}{2}\text{ln}(\frac{1+r_t}{1-r_t})$ in step 6 of Algorithm \ref{alg:boost} minimizes the training error. However, in our experiments, we observe empirically that a different choice of $\alpha$ works better for neural networks. In Fig. \ref{fig:alpha_comp}, we compare the CIFAR-10 test set performance of a CNN base classifier boosted using three choices of $\alpha$ for 10 rounds using the multi-class Adaboost described in Algorithm \ref{alg:boost}. We experiment with 1) $\alpha=\frac{1}{2}\text{log}(\frac{1+r_t}{1-r_t})$ originally proposed in \cite{SchapireSinger1999}, 2) $\alpha=\text{log}(\frac{(K-1)*r_t}{1-r_t})$ ($K$ is the number of classes) as proposed in \cite{Hastie2009MulticlassA} for another multi-class boosting algorithm SAMME and 3) $\alpha = \frac{1}{1+e^{-r_t}}$. From the figure, choosing $\alpha = \frac{1}{1+e^{-r_t}}$ performs better than $\frac{1}{2}\text{log}(\frac{1+r_t}{1-r_t})$ for higher rounds of boosting. Unless otherwise mentioned, we use $\frac{1}{1+e^{-r_t}}$ for AdaBoost for the rest of the experiments.

\begin{figure*}[h]   
\centering

\subfloat[]{\includegraphics[width=0.31\textwidth, keepaspectratio]{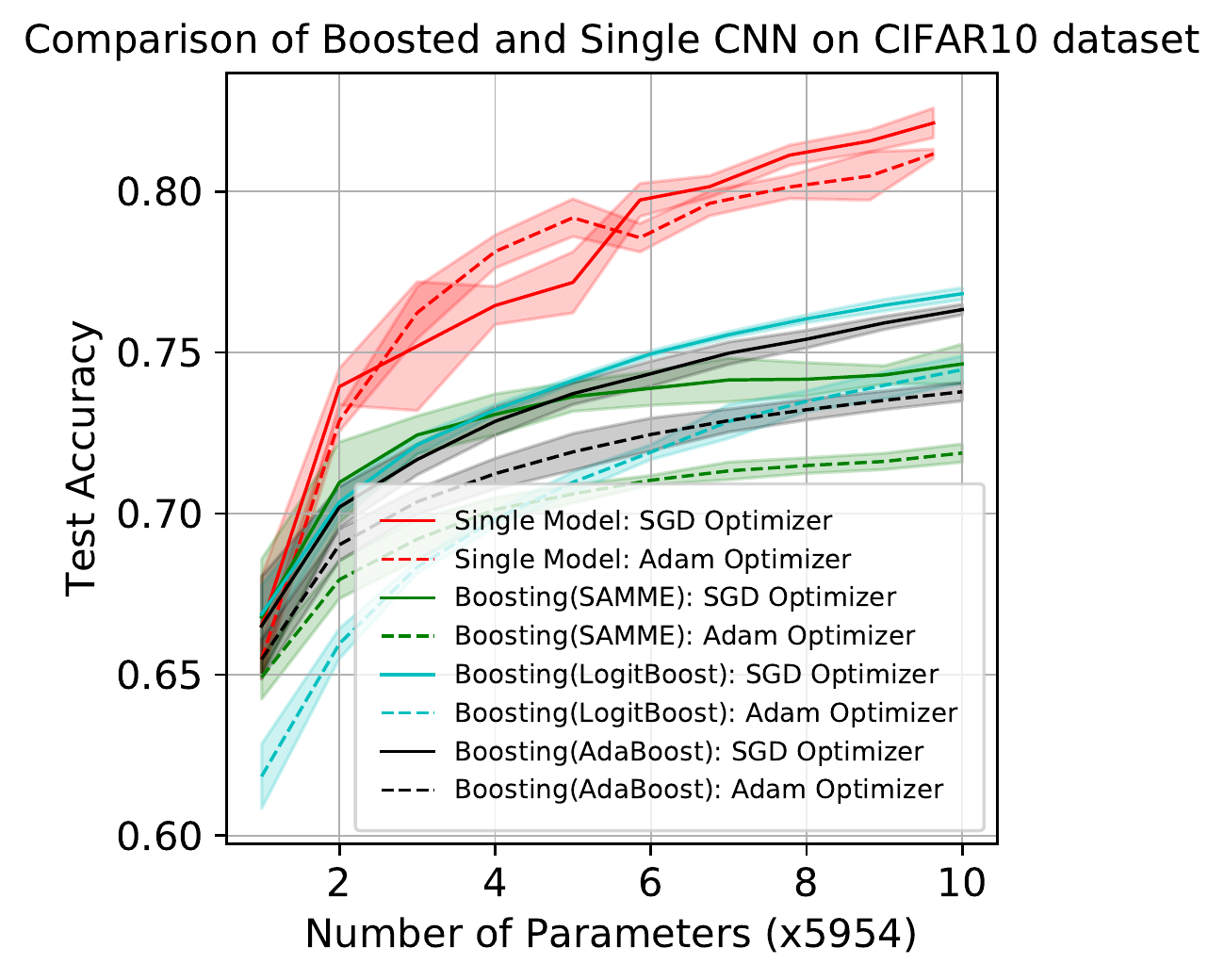}\label{fig:baby-cifar10}}
\subfloat[]{\includegraphics[width=0.32\textwidth, keepaspectratio]{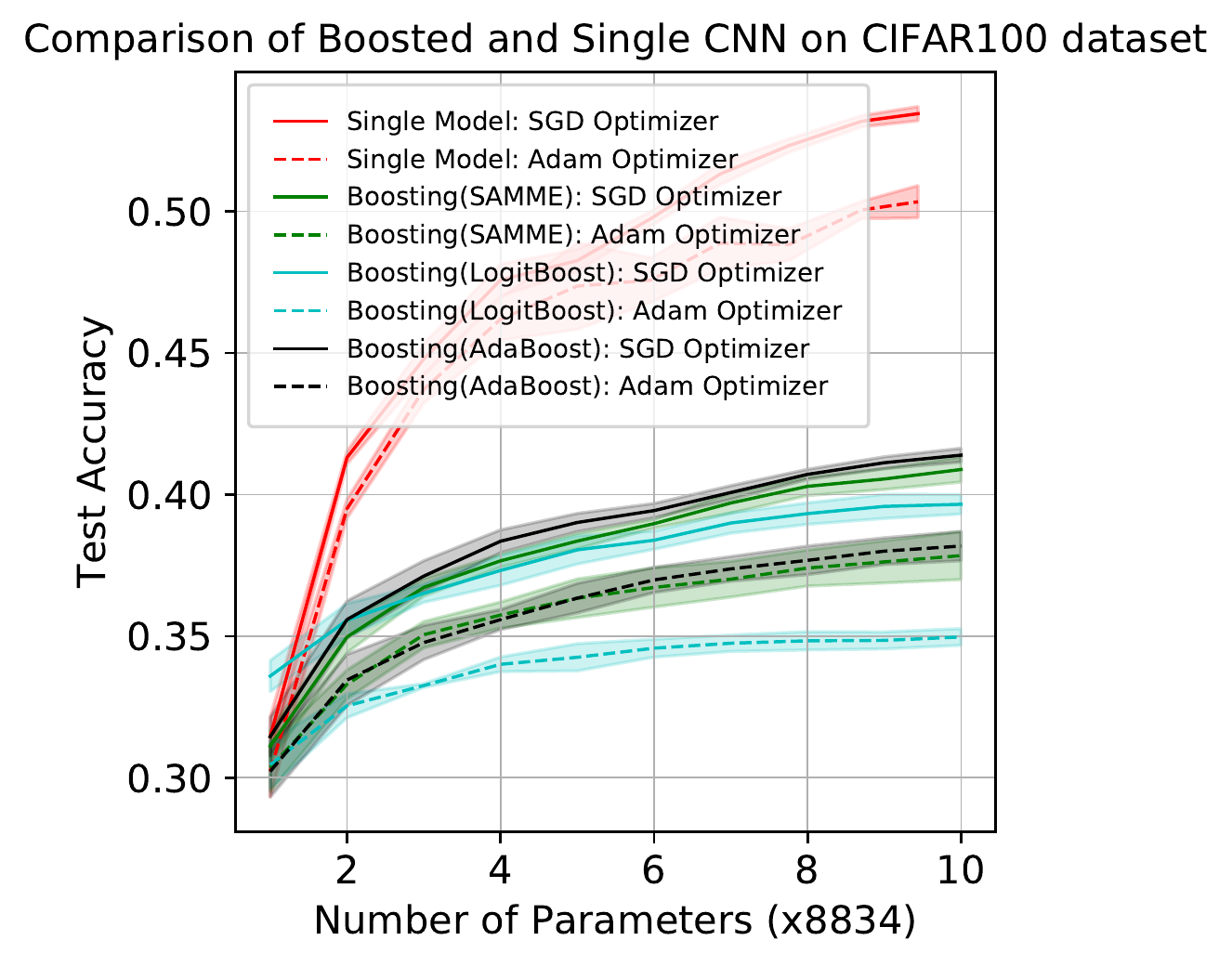}\label{fig:baby-cifar100}}
\subfloat[]{\includegraphics[width=0.30\textwidth, keepaspectratio]{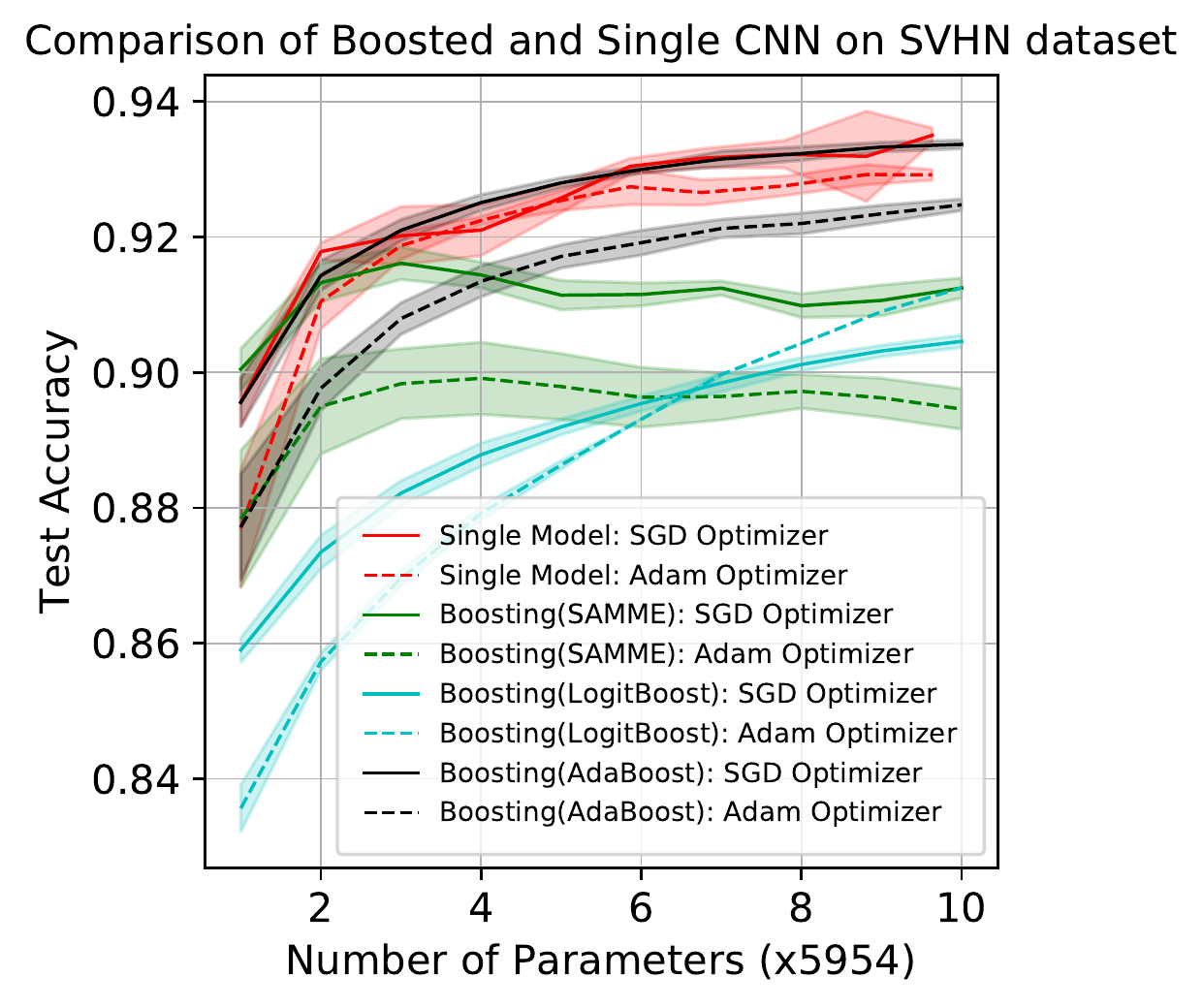}\label{fig:baby-svhn}}

\caption[Optional caption for list of figures 5-8]{Classification accuracy on CIFAR-10 [\ref{fig:baby-cifar10}], CIFAR-100 [\ref{fig:baby-cifar100}] and SVHN [\ref{fig:baby-svhn}] testing sets for single and boosted CNN classifiers versus number of parameters.}
\label{fig:baby}
\end{figure*}
\subsection{Main Experiments}
\textbf{Boosting Algorithms:}
In this section, we experiment with three boosting algorithms namely AdaBoost, SAMME~\cite{Hastie2009MulticlassA} and LogitBoost~\cite{FriedmanEtAl2000}. While we re-implement the exact algorithm of SAMME~\cite{Hastie2009MulticlassA} we had to make a few modifications to the LogitBoost algorithm~\cite{FriedmanEtAl2000} to work with a Neural Network base classifier. LogitBoost~\cite{FriedmanEtAl2000} requires training a regression model. We use the Mean Squared Error (MSE) loss for CIFAR-10 and SVHN datasets. However, for more than 10 classes, in the case of CIFAR-100, this training loss didn't converge. For CIFAR-100, we train a classification network using soft labels instead of one-hot-encoded targets. This can be achieved by minimizing the KL-Divergence between the soft labels and the network outputs. Additionally, for rounds greater than 1, we clamp the output of the networks to be no greater than 3. 

\subsubsection{CNN experiments}
Figure \ref{fig:baby} shows results comparing a series of single CNNs with their Adaboost, SAMME and LogitBoost ensemble counterparts on CIFAR-10, CIFAR-100 and SVHN using the SGD and ADAM optimizers. The solid lines are results with the SGD optimizer while the dotted lines are for ADAM. Red lines show results for a single network and black, green and cyan lines depict results of boosting using AdaBoost, SAMME and LogitBoost algorithms respectively.
Using AdaBoost, a single CNN classifier (with equal number of parameters) outperforms the boosted ensemble after 10 rounds of boosting by $5.8$, $12$ and $0.1$ percentage points on CIFAR-10, CIFAR-100 and SVHN, respectively, using an SGD optimizer. Using ADAM, the single classifier is better than the boosted ensemble by $7.5$, $14$ and $0.5$ percentage points on CIFAR-10, CIFAR-100 and SVHN, respectively.  

Using SAMME, a single CNN classifier (with equal number of parameters) outperforms the boosted ensemble after 10 rounds of boosting by $7.49$ and $12.57$ percentage points on CIFAR-10 and CIFAR-100 respectively using an SGD optimizer. Using ADAM, the single classifier is better than the boosted ensemble by $9.28$ and $12.50$ percentage points on CIFAR-10 and CIFAR-100 respectively. On the SVHN dataset, the SAMME algorithm fails after two rounds because neural network training fails to find a classifier with weighted error better than chance. We think this is due to the nature of the discrete margin computed in SAMME which puts almost all of the weight on the most difficult examples. 

Using LogitBoost, a single CNN classifier (with equal number of parameters) outperforms the boosted ensemble after 10 rounds of boosting by $5.30$, $13.79$ and $3.04$ percentage points on CIFAR-10, CIFAR-100 and SVHN, respectively, using an SGD optimizer. Using ADAM, the single classifier is better than the boosted ensemble by $6.70$, $15.37$ and $1.66$ percentage points on CIFAR-10, CIFAR-100 and SVHN, respectively.

\begin{figure*}[h]
\centering
\subfloat[]{\includegraphics[width=0.32\textwidth, keepaspectratio]{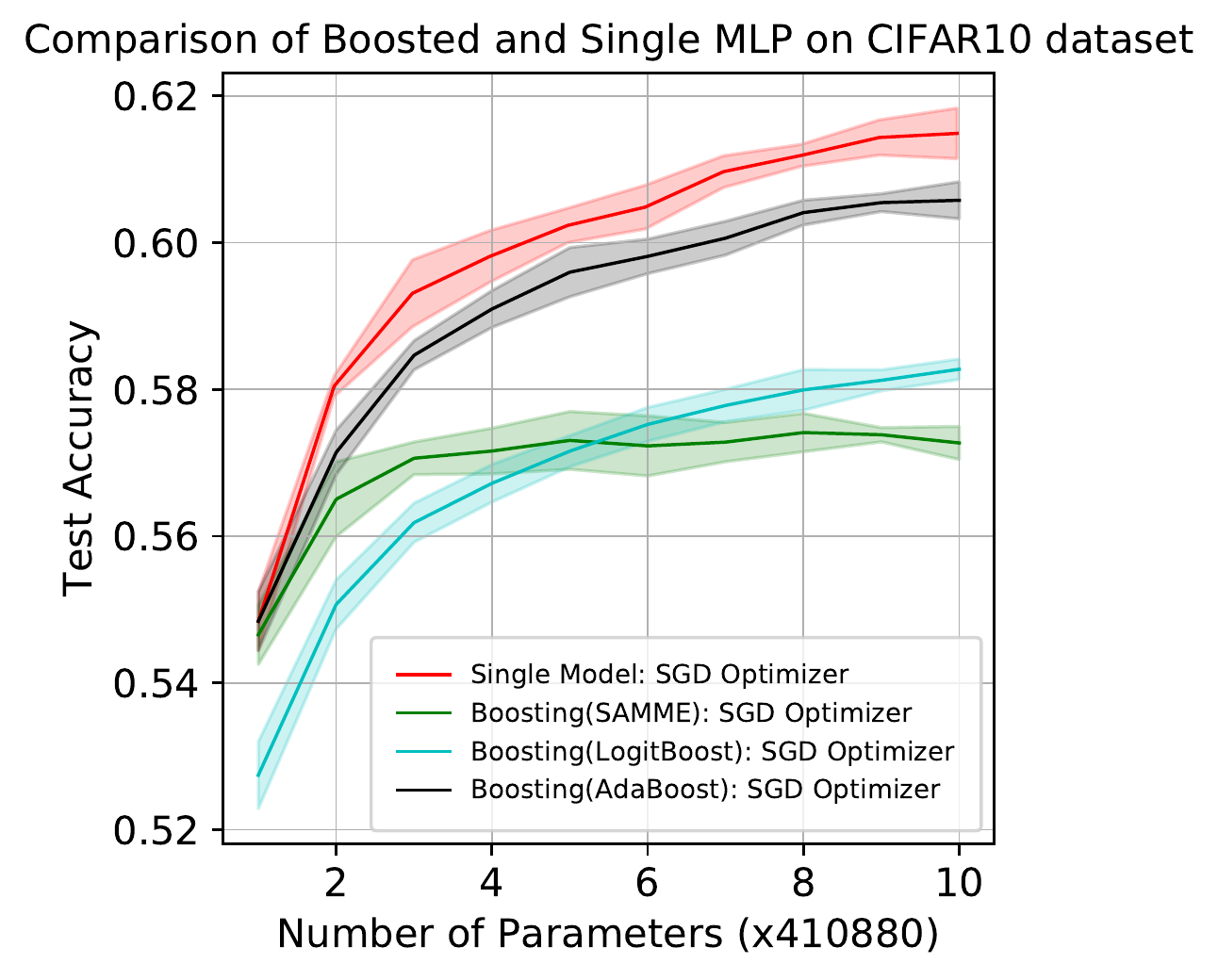}\label{fig:mlp-cifar10}}
\subfloat[]{\includegraphics[width=0.33\textwidth, keepaspectratio]{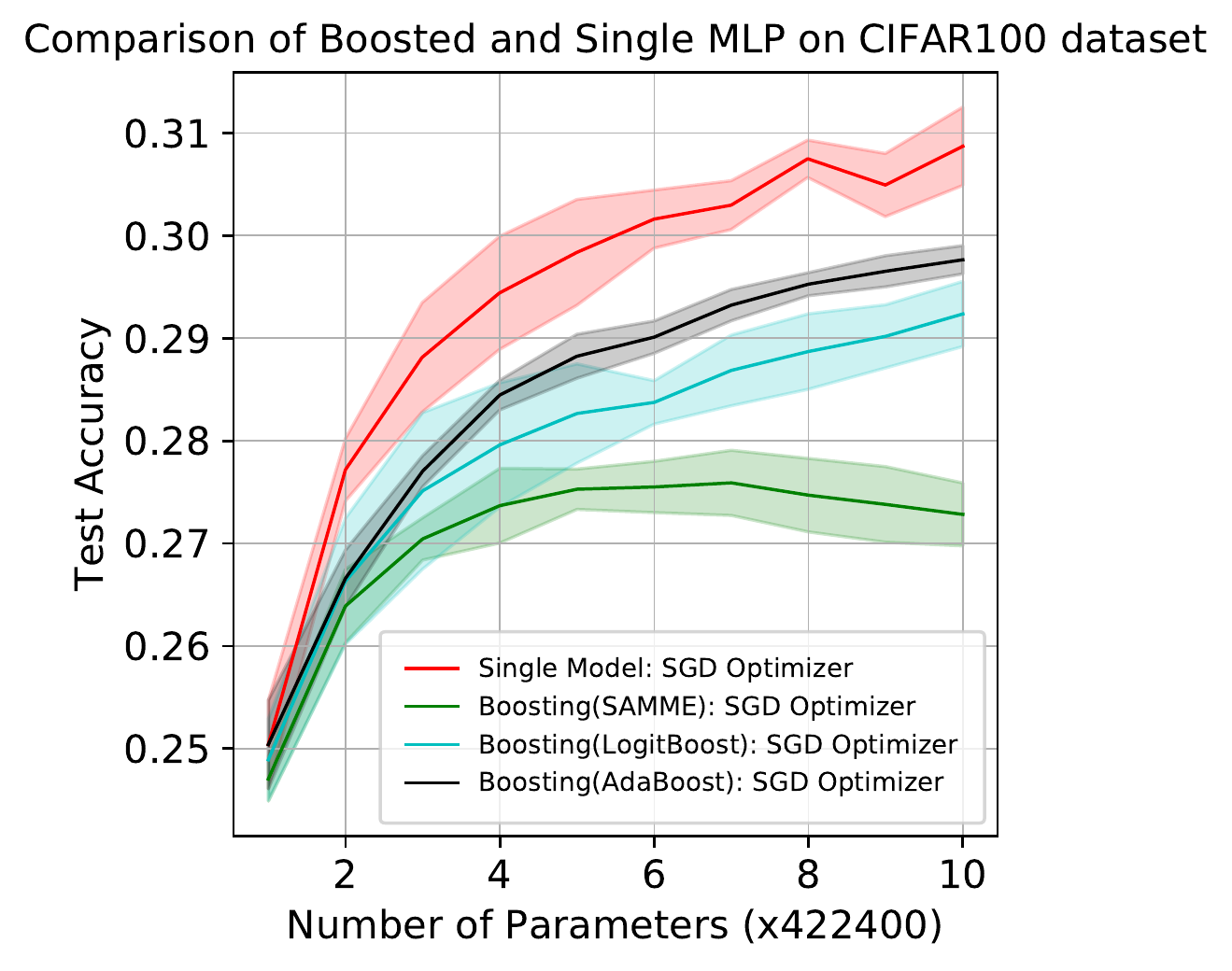}\label{fig:mlp-cifar100}}
\subfloat[]{\includegraphics[width=0.31\textwidth, keepaspectratio]{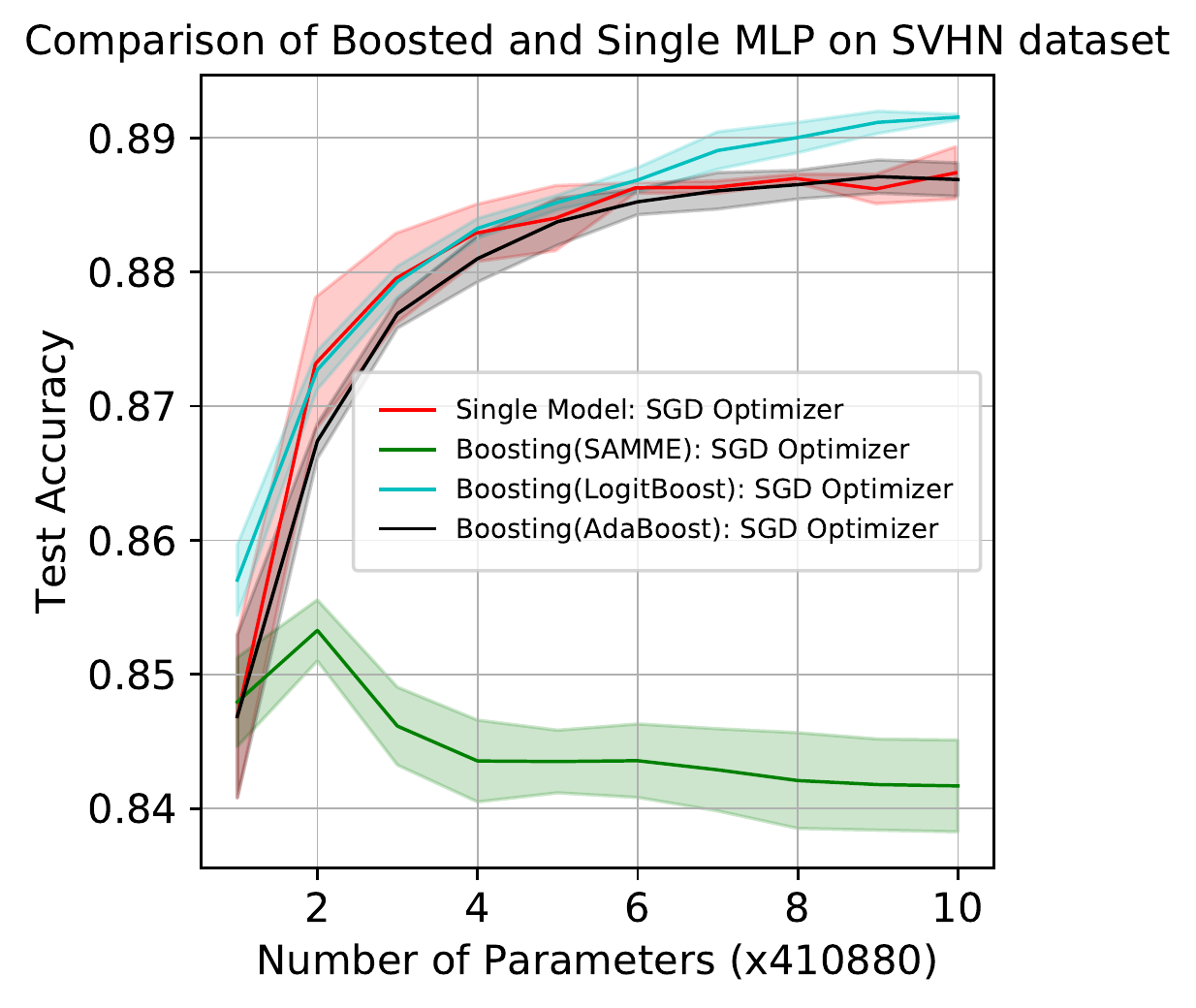}\label{fig:mlp-svhn}}

\caption[Optional caption for list of figures 5-8]{Classification accuracy on CIFAR-10 [\ref{fig:mlp-cifar10}], CIFAR-100 [\ref{fig:mlp-cifar100}] and SVHN [\ref{fig:mlp-svhn}] testing sets for a single and boosted MLP classifiers versus number of parameters.}
\label{fig:mlp}
\end{figure*}
\subsubsection{MLP experiments}
Figure \ref{fig:mlp} shows results of single MLPs and their boosted counterparts, trained using Adaboost, SAMME and LogitBoost, on each of the three test sets.  The single networks are more accurate than the equivalent boosted ensemble of networks for almost all values of the number of parameters.  After 10 rounds of boosting, using Adaboost, the equivalent single networks are about 1\%, 1.2\% and 0.2\% more accurate on CIFAR-10, CIFAR-100 and SVHN, respectively. Using SAMME, the single networks are about $4.22\%$ and $3.59\%$ more accurate than the boosted counterparts on CIFAR-10 and CIFAR-100 test sets respectively. Similar to the CNN architecture, the boosted MLP ensemble failed to converge on SVHN. Finally, using LogitBoost, the single networks are about $3.22\%$ and $1.64\%$ more accurate than the boosted counterparts on CIFAR-10 and CIFAR-100 test sets respectively. On SVHN, we observe that the ensemble trained using LogitBoost outperforms the single MLP architecture by 0.46\%. This is the only case out of 3 architectures and 3 datasets where a boosted ensemble outperformed a single neural network. We believe this is the case because the inductive bias of MLPs are not as suitable for structured data like images resulting in lower performance than the boosted ensembles. This can also be seen in the comparison of MLP and CNN single models on SVHN (88.74\% vs 93.50\%).

\begin{figure*}[h]   
\centering
\subfloat[]{\includegraphics[width=0.49\textwidth, keepaspectratio]{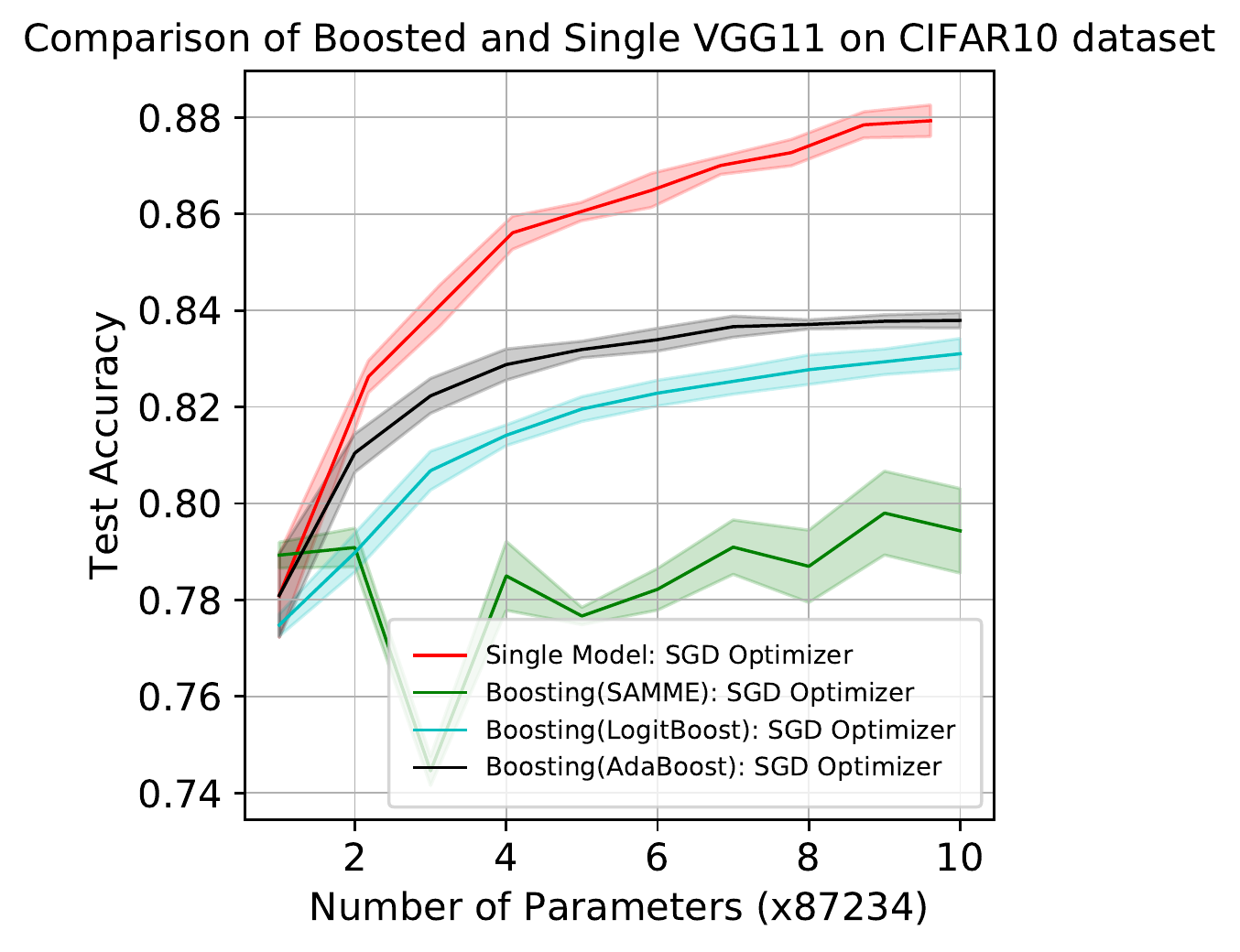}\label{fig:vgg8-cifar10}}
\subfloat[]{\includegraphics[width=0.5\textwidth, keepaspectratio]{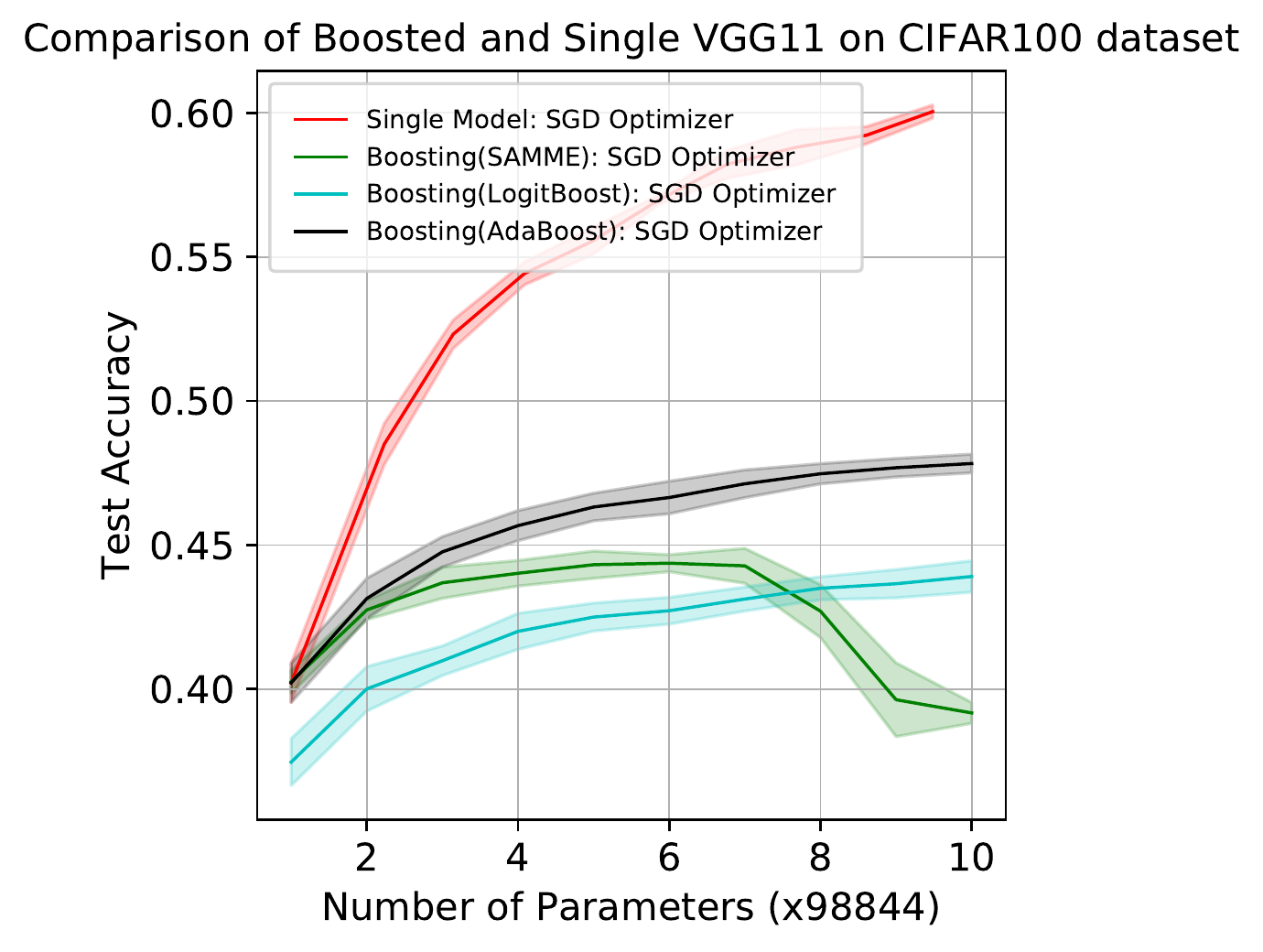}\label{fig:vgg8-cifar100}}
\caption[Optional caption for list of figures 5-8]{Classification accuracy on CIFAR-10 [\ref{fig:vgg8-cifar10}] and CIFAR-100 [\ref{fig:vgg8-cifar100}] testing sets for a VGG-8 single and boosted classifiers versus number of parameters.}
\label{fig:vgg8}
\end{figure*}

\textbf{VGG experiments: }Figure \ref{fig:vgg8} shows results comparing single VGG-8 classifiers with the boosted ensembles on CIFAR-10 and CIFAR-100. We do not show results for VGG-8 on SVHN as the single base classifier attains $100\%$ training accuracy, leaving little room for improvement for either boosting or larger networks.  It is clear from the plots that the single models outperform the boosted ensembles, trained using Adaboost, by a significant margin ($4$ and $12$ percentage points after 10 boosting rounds on CIFAR-10 and CIFAR-100, respectively). Using LogitBoost, the single models outperform the boosted ensemble by $4.83\%$ and $16.14\%$ points on CIFAR-10 and CIFAR-100 test sets. We observe a similar pattern with the boosted ensembles trained using SAMME algorithm as with the other two architectures above.


\textbf{Boosting rounds: }Finally, to verify that boosting a larger ensemble would not change the results, we boost a CNN base classifier for $50$ rounds ($297700$ parameters) on CIFAR-10 and compare it with a single VGG-8 classifier enlarged to $297680$ parameters.  The boosted ensemble achieves $77.00\%$ accuracy on the test set while the single network achieves $85.81\%$. We compare an ensemble of CNN classifiers to a single VGG-8 network because of the difficulty of increasing the number of parameters in the CNN base classifier by 50 times to create a single network.

\textbf{Final thoughts: }These experiments show that boosting neural networks, unlike boosting decision trees, does not lead to better accuracy compared to simply using a larger neural network.  One caveat to this statement that "bigger is better" for neural networks is that training set size and difficulty also matter.  In cases where a single neural network is able to achieve near 100\% accuracy on the training set, we observed that adding more parameters to the single network can lead to either overfitting or test error just not improving as much as it does using boosting.  In these cases, boosting is not a big improvement over training a single larger network, but neither is it true that a single large network is an improvement over a boosted ensemble.






\section{Discussion and Conclusions}
\label{sec:conclusions}


Through theoretical arguments and empirical evidence, we have demonstrated some limits of boosted ensembles of neural networks.  In particular, this research suggests that larger networks often yield higher accuracy than ensembles of smaller networks with the same total number of parameters.  This does not imply that boosting does not work with neural nets.  It does.  Boosted ensembles of neural nets have higher accuracy than any single neural net in the ensemble.  However, it is likely that a single, large neural net would have even greater accuracy.

An important reason for why boosted CNNs do not produce greater accuracy gains is because linear combinations of CNNs do not add any new expressive power over a single (larger) CNN, unlike boosted decision trees.  Another way to view this fact is to think of a single CNN as an ensemble in itself.  To see this, consider a
convolution layer with $K$ input channels, and $D$ filters ($w_d$, $d\in\{1,2,\cdots, D\}$).  The output of the layer can be written as
\begin{equation}
    O[d] = \sum_{k=1}^{K} w_d[k] * I[k] \quad d \in \{1,2,\cdots, D\}
\end{equation}
where $*$ is the convolution operator, $[.]$ is the channel slicing, $w_d$ is the $d^{th}$ filter and $O$, $I$ are the output and input of the layer respectively. This is an ensemble over feature channels analogous to the score ensemble used in boosting.
In a similar fashion, the matrix vector product in multilayer perceptrons (MLP) (or a fully connected layer in a CNN) can be viewed as an ensemble over features. 
This view of neural networks as ensembles may provide another way of understanding the robustness of overparameterized networks against overfitting.  The same theory that has been developed for understanding the robustness of classifier ensembles against overfitting can be applied to single neural networks.

\vspace{.1in}
\noindent
This research did not receive any specific grant from funding agencies in the public, commercial, or not-for-profit sectors.

\bibliography{main}

\begin{thebibliography}{41}
\providecommand{\natexlab}[1]{#1}
\providecommand{\url}[1]{\texttt{#1}}
\providecommand{\href}[2]{#2}
\providecommand{\path}[1]{#1}
\providecommand{\DOIprefix}{doi:}
\providecommand{\ArXivprefix}{arXiv:}
\providecommand{\URLprefix}{URL: }
\providecommand{\Pubmedprefix}{pmid:}
\providecommand{\doi}[1]{\href{http://dx.doi.org/#1}{\path{#1}}}
\providecommand{\Pubmed}[1]{\href{pmid:#1}{\path{#1}}}
\providecommand{\BIBand}{and}
\providecommand{\bibinfo}[2]{#2}
\ifx\xfnm\undefined \def\xfnm[#1]{\unskip,\space#1}\fi
\makeatletter\def\@biblabel#1{#1.}\makeatother
\bibitem[{Schapire(1990)}]{Schapire1990}
\bibinfo{author}{Schapire\xfnm[ R.]}.
\newblock \bibinfo{title}{The strength of weak learnability}.
\newblock \emph{\bibinfo{journal}{Machine Learning}}
  \bibinfo{year}{1990};\bibinfo{volume}{5}(\bibinfo{number}{2}).
\bibitem[{Freund(1995)}]{Freund1995}
\bibinfo{author}{Freund\xfnm[ Y.]}.
\newblock \bibinfo{title}{Boosting a weak learning algorithm by majority}.
\newblock \emph{\bibinfo{journal}{Information and Computation}}
  \bibinfo{year}{1995};\bibinfo{volume}{12}(\bibinfo{number}{2}).
\bibitem[{Friedman et~al.(2000)Friedman, Hastie and
  Tibshirani}]{FriedmanEtAl2000}
\bibinfo{author}{Friedman\xfnm[ J.]}, \bibinfo{author}{Hastie\xfnm[ T.]},
  \bibinfo{author}{Tibshirani\xfnm[ R.]}.
\newblock \bibinfo{title}{Additive logistic regression: a statistical view of
  boosting}.
\newblock \emph{\bibinfo{journal}{Annals of Statistics}}
  \bibinfo{year}{2000};\bibinfo{volume}{28}(\bibinfo{number}{2}).
\bibitem[{Freund and Schapire(1997)}]{FreundSchapire1997}
\bibinfo{author}{Freund\xfnm[ Y.]}, \bibinfo{author}{Schapire\xfnm[ R.]}.
\newblock \bibinfo{title}{A decision-theoretic generalization of on-line
  learning and an application to boosting}.
\newblock \emph{\bibinfo{journal}{Journal of Computer and System Sciences}}
  \bibinfo{year}{1997};\bibinfo{volume}{55}.
\bibitem[{Schwenk and Bengio(2000)}]{SchwenkBengio2000}
\bibinfo{author}{Schwenk\xfnm[ H.]}, \bibinfo{author}{Bengio\xfnm[ Y.]}.
\newblock \bibinfo{title}{Boosting neural networks}.
\newblock \emph{\bibinfo{journal}{Neural Computation}}
  \bibinfo{year}{2000};\bibinfo{volume}{12}(\bibinfo{number}{8}):\bibinfo{pages}{1869–1887}.
\bibitem[{Medera and Babinec(2009)}]{MederaBabinec2009}
\bibinfo{author}{Medera\xfnm[ D.]}, \bibinfo{author}{Babinec\xfnm[ S.]}.
\newblock \bibinfo{title}{Incremental learning of convolutional neural
  networks}.
\newblock In: \emph{\bibinfo{booktitle}{International Joint Conference on
  Computational Intelligence (IJCCI)}}. \bibinfo{year}{2009}:\unskip.
\bibitem[{Mosca and Magoulas(2016)}]{Deep_Incremental_Boosting}
\bibinfo{author}{Mosca\xfnm[ A.]}, \bibinfo{author}{Magoulas\xfnm[ G.]}.
\newblock \bibinfo{title}{Deep incremental boosting}.
\newblock In: \emph{\bibinfo{booktitle}{2nd Global Conference on Artificial
  Intelligence (GCAI)}}; vol.~\bibinfo{volume}{41} of
  \emph{\bibinfo{series}{EPiC Series in Computing}}.
  \bibinfo{year}{2016}:\unskip \bibinfo{pages}{293--302}.
\bibitem[{Moghimi et~al.(2016)Moghimi, Saberian, Yang, Li, Vasconcelos and
  Belongie}]{MoghimiEtAl2016}
\bibinfo{author}{Moghimi\xfnm[ M.]}, \bibinfo{author}{Saberian\xfnm[ M.]},
  \bibinfo{author}{Yang\xfnm[ J.]}, \bibinfo{author}{Li\xfnm[ L.J.]},
  \bibinfo{author}{Vasconcelos\xfnm[ N.]}, \bibinfo{author}{Belongie\xfnm[
  S.]}.
\newblock \bibinfo{title}{Boosted convolutional neural networks}.
\newblock In: \emph{\bibinfo{booktitle}{Proceedings of the British Machine
  Vision Conference (BMVC)}}. \bibinfo{year}{2016}:\unskip.
\bibitem[{Freund and Schapire(1996)}]{FreundSchapire1996}
\bibinfo{author}{Freund\xfnm[ Y.]}, \bibinfo{author}{Schapire\xfnm[ R.]}.
\newblock \bibinfo{title}{Experiments with a new boosting algorithm}.
\newblock In: \emph{\bibinfo{booktitle}{International Conference on Machine
  Learning (ICML)}}. \bibinfo{year}{1996}:\unskip.
\bibitem[{Friedman(1999)}]{Friedman1999}
\bibinfo{author}{Friedman\xfnm[ J.H.]}.
\newblock \bibinfo{title}{Greedy function approximation: A gradient boosting
  machine}.
\newblock \emph{\bibinfo{journal}{Annals of Statistics}}
  \bibinfo{year}{1999};\bibinfo{volume}{29}(\bibinfo{number}{5}).
\bibitem[{Mason et~al.(1999)Mason, Baxter, Bartlett and Frean}]{MasonEtAl1999}
\bibinfo{author}{Mason\xfnm[ L.]}, \bibinfo{author}{Baxter\xfnm[ J.]},
  \bibinfo{author}{Bartlett\xfnm[ P.]}, \bibinfo{author}{Frean\xfnm[ M.]}.
\newblock \bibinfo{title}{Boosting algorithms as gradient descent}.
\newblock In: \emph{\bibinfo{booktitle}{Advances in Neural Information
  Processing Systems (NIPS)}}. \bibinfo{year}{1999}:\unskip.
\bibitem[{Zhu et~al.(2006)Zhu, Rosset, Zou and Hastie}]{ZhuEtAl2006}
\bibinfo{author}{Zhu\xfnm[ J.]}, \bibinfo{author}{Rosset\xfnm[ S.]},
  \bibinfo{author}{Zou\xfnm[ H.]}, \bibinfo{author}{Hastie\xfnm[ T.]}.
\newblock \bibinfo{title}{Multi-class boosting}.
\newblock \emph{\bibinfo{journal}{Statistics and its Interface}}
  \bibinfo{year}{2006};\bibinfo{volume}{2}(\bibinfo{number}{3}).
\bibitem[{Quinlan(1996)}]{Quinlan1996}
\bibinfo{author}{Quinlan\xfnm[ R.]}.
\newblock \bibinfo{title}{Bagging, boosting and c4.5}.
\newblock In: \emph{\bibinfo{booktitle}{Proceedings of the Thirteenth National
  Conference on Artificial Intelligence (AAAI)}}. \bibinfo{year}{1996}:\unskip.
\bibitem[{Schapire and Singer(1999)}]{SchapireSinger1999}
\bibinfo{author}{Schapire\xfnm[ R.]}, \bibinfo{author}{Singer\xfnm[ Y.]}.
\newblock \bibinfo{title}{Improved boosting algorithms using confidence-rated
  predictions}.
\newblock \emph{\bibinfo{journal}{Machine Learning}}
  \bibinfo{year}{1999};\bibinfo{volume}{37}:\bibinfo{pages}{297–336}.
\bibitem[{Breiman et~al.(1984)Breiman, Friedman, Olshen and
  Stone}]{BreimanEtAl1984}
\bibinfo{author}{Breiman\xfnm[ L.]}, \bibinfo{author}{Friedman\xfnm[ J.]},
  \bibinfo{author}{Olshen\xfnm[ R.]}, \bibinfo{author}{Stone\xfnm[ C.]}.
\newblock \bibinfo{title}{Classification and Regression Trees}.
\newblock \bibinfo{address}{Monterey, CA}: \bibinfo{publisher}{Wadsworth \&
  Brooks/Cole Advanced Books \& Software}; \bibinfo{year}{1984}.
\bibitem[{Drucker et~al.(1993)Drucker, Schapire and Simard}]{DruckerEtAl1993}
\bibinfo{author}{Drucker\xfnm[ H.]}, \bibinfo{author}{Schapire\xfnm[ R.]},
  \bibinfo{author}{Simard\xfnm[ P.]}.
\newblock \bibinfo{title}{Improving performance in neural networks using a
  boosting algorithm}.
\newblock In: \emph{\bibinfo{booktitle}{Proceedings of the International
  Conference on Neural Information Processing Systems (NIPS)}}.
  \bibinfo{year}{1993}:\unskip.
\bibitem[{Schwenk and Bengio(1997)}]{SchwenkBengio1997}
\bibinfo{author}{Schwenk\xfnm[ H.]}, \bibinfo{author}{Bengio\xfnm[ Y.]}.
\newblock \bibinfo{title}{Adaboosting neural networks: Application to on-line
  character recognition}.
\newblock In: \emph{\bibinfo{booktitle}{International Conference on Artificial
  Neural Networks (ICANN)}}. \bibinfo{year}{1997}:\unskip.
\bibitem[{Banfield et~al.(2007)Banfield, Hall, Bowyer and
  Kegelmeyer}]{BanfieldEtAl2007}
\bibinfo{author}{Banfield\xfnm[ R.]}, \bibinfo{author}{Hall\xfnm[ L.]},
  \bibinfo{author}{Bowyer\xfnm[ K.]}, \bibinfo{author}{Kegelmeyer\xfnm[ W.P.]}.
\newblock \bibinfo{title}{A comparison of decision tree ensemble creation
  techniques}.
\newblock \emph{\bibinfo{journal}{IEEE Transactions on Pattern Analysis and
  Machine Intelligence (PAMI)}} \bibinfo{year}{2007};\bibinfo{volume}{29}.
\bibitem[{Lee et~al.(2015)Lee, Purushwalkam, Cogswell, Crandall and
  Batra}]{LeeEtAl2015}
\bibinfo{author}{Lee\xfnm[ S.]}, \bibinfo{author}{Purushwalkam\xfnm[ S.]},
  \bibinfo{author}{Cogswell\xfnm[ M.]}, \bibinfo{author}{Crandall\xfnm[ D.]},
  \bibinfo{author}{Batra\xfnm[ D.]}.
\newblock \bibinfo{title}{Why m heads are better than one: Training a diverse
  ensemble of deep networks}.
\newblock \emph{\bibinfo{journal}{arXiv preprint arXiv:151106314 [csCV]}}
  \bibinfo{year}{2015};.
\bibitem[{Mosca and Magoulas(2017)}]{MoscaMagoulas2017}
\bibinfo{author}{Mosca\xfnm[ A.]}, \bibinfo{author}{Magoulas\xfnm[ G.D.]}.
\newblock \bibinfo{title}{Boosted residual networks}.
\newblock In: \emph{\bibinfo{booktitle}{Engineering Applications of Neural
  Networks}}. \bibinfo{publisher}{Springer International Publishing};
  \bibinfo{year}{2017}:\unskip \bibinfo{pages}{137--148}.
\bibitem[{Chen and Guestrin(2016)}]{ChenGuestrin2016}
\bibinfo{author}{Chen\xfnm[ T.]}, \bibinfo{author}{Guestrin\xfnm[ C.]}.
\newblock \bibinfo{title}{Xgboost: A scalable tree boosting system}.
\newblock In: \emph{\bibinfo{booktitle}{22nd SIGKDD Conference on Knowledge
  Discovery and Data Mining}}. \bibinfo{year}{2016}:\unskip.
\bibitem[{Dorogush et~al.(2017)Dorogush, Ershov and Gulin}]{DorogushEtAl2017}
\bibinfo{author}{Dorogush\xfnm[ A.V.]}, \bibinfo{author}{Ershov\xfnm[ V.]},
  \bibinfo{author}{Gulin\xfnm[ A.]}.
\newblock \bibinfo{title}{Catboost: gradient boosting with categorical features
  support}.
\newblock In: \emph{\bibinfo{booktitle}{Workshop on ML Systems at NIPS}}.
  \bibinfo{year}{2017}:\unskip.
\bibitem[{He et~al.(2019)He, Lin and Wu}]{HeEtAl2019}
\bibinfo{author}{He\xfnm[ Z.]}, \bibinfo{author}{Lin Danchen ad~Lau\xfnm[ T.]},
  \bibinfo{author}{Wu\xfnm[ M.]}.
\newblock \bibinfo{title}{Gradient boosting machine: A survey}.
\newblock \emph{\bibinfo{journal}{arXiv preprint arXiv:190806951 [statML]}}
  \bibinfo{year}{2019};.
\bibitem[{Schapire(2003)}]{Schapire2003}
\bibinfo{author}{Schapire\xfnm[ R.]}.
\newblock \bibinfo{title}{The boosting approach to machine learning: An
  overview}.
\newblock \emph{\bibinfo{journal}{Nonlinear Estimation and Classification,
  Lecture Notes in Statistics}} \bibinfo{year}{2003};\bibinfo{volume}{171}.
\bibitem[{Allen-Zhu et~al.(2019)Allen-Zhu, Li and Song}]{AllenZhuEtAl2019}
\bibinfo{author}{Allen-Zhu\xfnm[ Z.]}, \bibinfo{author}{Li\xfnm[ Y.]},
  \bibinfo{author}{Song\xfnm[ Z.]}.
\newblock \bibinfo{title}{A convergence theory for deep learning via
  over-parameterization}.
\newblock In: \emph{\bibinfo{booktitle}{International Conference on Machine
  Learning (ICML)}}. \bibinfo{year}{2019}:\unskip.
\bibitem[{Belkin et~al.(2019)Belkin, Hsu, Ma and Mandal}]{BelkinEtAl2019}
\bibinfo{author}{Belkin\xfnm[ M.]}, \bibinfo{author}{Hsu\xfnm[ D.]},
  \bibinfo{author}{Ma\xfnm[ S.]}, \bibinfo{author}{Mandal\xfnm[ S.]}.
\newblock \bibinfo{title}{Reconciling modern machine learning practice and the
  classical bias-variance trade-off}.
\newblock \emph{\bibinfo{journal}{Proceedings of the National Academy of
  Sciences}} \bibinfo{year}{2019};.
\bibitem[{Zhang et~al.(2017)Zhang, Bengio, Hardt, Recht and
  Vinyals}]{ZhangEtAl2017}
\bibinfo{author}{Zhang\xfnm[ C.]}, \bibinfo{author}{Bengio\xfnm[ S.]},
  \bibinfo{author}{Hardt\xfnm[ M.]}, \bibinfo{author}{Recht\xfnm[ B.]},
  \bibinfo{author}{Vinyals\xfnm[ O.]}.
\newblock \bibinfo{title}{Understanding deep learning requires rethinking
  generalization}.
\newblock In: \emph{\bibinfo{booktitle}{International Conference on Learning
  Representations (ICLR)}}. \bibinfo{year}{2017}:\unskip.
\bibitem[{Poggio et~al.(2018)Poggio, Kawaguchi, Liao, Miranda, Rosasco, Boix,
  Hidary and Mhaskar}]{PoggioEtAl2018}
\bibinfo{author}{Poggio\xfnm[ T.]}, \bibinfo{author}{Kawaguchi\xfnm[ K.]},
  \bibinfo{author}{Liao\xfnm[ Q.]}, \bibinfo{author}{Miranda\xfnm[ B.]},
  \bibinfo{author}{Rosasco\xfnm[ L.]}, \bibinfo{author}{Boix\xfnm[ X.]},
  \bibinfo{author}{Hidary\xfnm[ J.]}, \bibinfo{author}{Mhaskar\xfnm[ H.]}.
\newblock \bibinfo{title}{Theory of deep learning iii: explaining the
  non-overfitting puzzle}.
\newblock \emph{\bibinfo{journal}{MIT Center for Brains, Minds and Machines,
  Memo No 073}} \bibinfo{year}{2018};.
\bibitem[{Saberian and Vasconcelos(2011)}]{SaberianVasconcelos2011}
\bibinfo{author}{Saberian\xfnm[ M.J.]}, \bibinfo{author}{Vasconcelos\xfnm[
  N.]}.
\newblock \bibinfo{title}{Multiclass boosting: Theory and algorithms}.
\newblock In: \emph{\bibinfo{booktitle}{Proceedings of the 24th International
  Conference on Neural Information Processing Systems (NIPS)}}.
  \bibinfo{year}{2011}:\unskip \bibinfo{pages}{2124–2132}.
\bibitem[{Quinlan(1993)}]{Quinlan1993}
\bibinfo{author}{Quinlan\xfnm[ J.R.]}.
\newblock \bibinfo{title}{C4.5: Programs for Machine Learning}.
\newblock \bibinfo{publisher}{Morgan Kaufman Publishers}; \bibinfo{year}{1993}.
\bibitem[{Schapire et~al.(1998)Schapire, Freund, Bartlett and
  Lee}]{SchapireEtAl1998}
\bibinfo{author}{Schapire\xfnm[ R.]}, \bibinfo{author}{Freund\xfnm[ Y.]},
  \bibinfo{author}{Bartlett\xfnm[ P.]}, \bibinfo{author}{Lee\xfnm[ W.S.]}.
\newblock \bibinfo{title}{Boosting the margin: A new explanation for the
  effectiveness of voting methods}.
\newblock \emph{\bibinfo{journal}{Annals of Statistics}}
  \bibinfo{year}{1998};\bibinfo{volume}{26}(\bibinfo{number}{5}).
\bibitem[{Reyzin and Schapire(2006)}]{ReyzinSchapire2006}
\bibinfo{author}{Reyzin\xfnm[ L.]}, \bibinfo{author}{Schapire\xfnm[ R.]}.
\newblock \bibinfo{title}{How boosting the margin can also boost classifier
  complexity}.
\newblock In: \emph{\bibinfo{booktitle}{International Conference on Machine
  Learning (ICML)}}. \bibinfo{year}{2006}:\unskip.
\bibitem[{Gao and Zhou(2013)}]{GaoZhou2013}
\bibinfo{author}{Gao\xfnm[ W.]}, \bibinfo{author}{Zhou\xfnm[ Z.H.]}.
\newblock \bibinfo{title}{On the doubt about margin explanation of boosting}.
\newblock \emph{\bibinfo{journal}{Artificial Intelligence}}
  \bibinfo{year}{2013};.
\bibitem[{Viola and Jones(2004)}]{ViolaJones2004}
\bibinfo{author}{Viola\xfnm[ P.]}, \bibinfo{author}{Jones\xfnm[ M.]}.
\newblock \bibinfo{title}{Robust real-time face detection}.
\newblock \emph{\bibinfo{journal}{International Journal of Computer Vision}}
  \bibinfo{year}{2004};.
\bibitem[{Viola et~al.(2003)Viola, Jones and Snow}]{ViolaEtAl2003}
\bibinfo{author}{Viola\xfnm[ P.]}, \bibinfo{author}{Jones\xfnm[ M.]},
  \bibinfo{author}{Snow\xfnm[ D.]}.
\newblock \bibinfo{title}{Detecting pedestrians using patterns of motion and
  appearance}.
\newblock In: \emph{\bibinfo{booktitle}{IEEE International Conference on
  Computer Vision (ICCV)}}. \bibinfo{year}{2003}:\unskip.
\bibitem[{Bourdev and Brandt(2005)}]{BourdevBrandt2005}
\bibinfo{author}{Bourdev\xfnm[ L.]}, \bibinfo{author}{Brandt\xfnm[ J.]}.
\newblock \bibinfo{title}{Robust object detection via soft cascade}.
\newblock In: \emph{\bibinfo{booktitle}{IEEE Conference on Computer Vision and
  Pattern Recognition (CVPR)}}. \bibinfo{year}{2005}:\unskip.
\bibitem[{Lecun et~al.(1998)Lecun, Bottou, Bengio and
  Haffner}]{Lecun98gradient-basedlearning}
\bibinfo{author}{Lecun\xfnm[ Y.]}, \bibinfo{author}{Bottou\xfnm[ L.]},
  \bibinfo{author}{Bengio\xfnm[ Y.]}, \bibinfo{author}{Haffner\xfnm[ P.]}.
\newblock \bibinfo{title}{Gradient-based learning applied to document
  recognition}.
\newblock In: \emph{\bibinfo{booktitle}{Proceedings of the IEEE}}.
  \bibinfo{year}{1998}:\unskip \bibinfo{pages}{2278--2324}.
\bibitem[{Krizhevsky(2009)}]{Krizhevsky09learningmultiple}
\bibinfo{author}{Krizhevsky\xfnm[ A.]}.
\newblock \bibinfo{title}{Learning multiple layers of features from tiny
  images}.
\newblock \bibinfo{type}{Tech. Rep.}; \bibinfo{year}{2009}.
\bibitem[{Netzer et~al.(2011)Netzer, Wang, Coates, Bissacco, Wu and Ng}]{37648}
\bibinfo{author}{Netzer\xfnm[ Y.]}, \bibinfo{author}{Wang\xfnm[ T.]},
  \bibinfo{author}{Coates\xfnm[ A.]}, \bibinfo{author}{Bissacco\xfnm[ A.]},
  \bibinfo{author}{Wu\xfnm[ B.]}, \bibinfo{author}{Ng\xfnm[ A.Y.]}.
\newblock \bibinfo{title}{Reading digits in natural images with unsupervised
  feature learning}.
\newblock In: \emph{\bibinfo{booktitle}{NIPS Workshop on Deep Learning and
  Unsupervised Feature Learning}}. \bibinfo{year}{2011}:\unskip.
\bibitem[{Simonyan and Zisserman(2015)}]{Simonyan15}
\bibinfo{author}{Simonyan\xfnm[ K.]}, \bibinfo{author}{Zisserman\xfnm[ A.]}.
\newblock \bibinfo{title}{Very deep convolutional networks for large-scale
  image recognition}.
\newblock In: \emph{\bibinfo{booktitle}{International Conference on Learning
  Representations (ICLR)}}. \bibinfo{year}{2015}:\unskip.
\bibitem[{Hastie et~al.(2009)Hastie, Rosset, Zhu and
  Zou}]{Hastie2009MulticlassA}
\bibinfo{author}{Hastie\xfnm[ T.]}, \bibinfo{author}{Rosset\xfnm[ S.]},
  \bibinfo{author}{Zhu\xfnm[ J.]}, \bibinfo{author}{Zou\xfnm[ H.]}.
\newblock \bibinfo{title}{Multi-class adaboost ∗}.
\newblock \emph{\bibinfo{journal}{Statistics and Its Interface}}
  \bibinfo{year}{2009};\bibinfo{volume}{2}:\bibinfo{pages}{349--360}.

\end{thebibliography}
\end{document}